\documentclass{article}

\renewcommand{\sfdefault}{phv}
\newcommand{\dif}{\mathrm{d}}
\newcommand{\E}{\mathop{\mathbb{E}}\displaylimits}

\newcommand{\tr}{\mathop{\mathrm{tr}}}

\usepackage{algpseudocode}
\usepackage{bm, mathtools}
\usepackage{thm-restate}
\usepackage{etoolbox}

\newbool{sketch}

\def\eqref#1{equation~\ref{#1}}
\def\1{\bm{1}}

\def\vzero{{\bm{0}}}

\def\vmu{{\bm{\mu}}}

\def\vvarepsilon{{\bm{\varepsilon}}}
\def\vdelta{{\bm{\delta}}}

\def\vm{{\bm{m}}}

\def\vp{{\bm{p}}}

\def\vu{{\bm{u}}}

\def\vx{{\bm{x}}}
\def\vy{{\bm{y}}}

\def\mA{{\bm{A}}}
\def\mB{{\bm{B}}}

\def\mE{{\bm{E}}}

\def\mI{{\bm{I}}}

\def\mQ{{\bm{Q}}}
\def\mR{{\bm{R}}}
\def\mS{{\bm{S}}}

\def\mW{{\bm{W}}}
\def\mX{{\bm{X}}}

\def\mLambda{{\bm{\Lambda}}}
\def\mSigma{{\bm{\Sigma}}}

\DeclareMathAlphabet{\mathsfit}{\encodingdefault}{\sfdefault}{m}{sl}
\SetMathAlphabet{\mathsfit}{bold}{\encodingdefault}{\sfdefault}{bx}{n}

\newcommand{\Var}{\mathop{\mathrm{Var}}}
\newcommand{\Cov}{\mathop{\mathrm{Cov}}}

\usepackage{auxiliaries/PRIMEarxiv}
\title{Convergence Dynamics and Stabilization Strategies of Co-Evolving Generative Models}

\author{%
Weiguo Gao\textsuperscript{\(\dagger\)\(\ddagger\)*}, Ming Li\textsuperscript{\(\dagger\)}\\
\textsuperscript{\(\dagger\)}School of Mathematical Sciences, Fudan University, Shanghai, 200433, China\\
\textsuperscript{\(\ddagger\)}School of Data Science, Fudan University, Shanghai, 200433, China\\
\textsuperscript{*}Shanghai Key Laboratory of Contemporary Applied Mathematics, Shanghai, 200433, China\\
\texttt{wggao@fudan.edu.cn}, \texttt{mingli23@m.fudan.edu.cn}
}

\setbool{sketch}{false}

\begin{document}

\maketitle

% REQUIRED
\begin{abstract}
The increasing prevalence of synthetic data in training loops has raised concerns about model collapse, where generative models degrade when trained on their own outputs. While prior work focuses on this self-consuming process, we study an underexplored yet prevalent phenomenon: co-evolving generative models that shape each other's training through iterative feedback. This is common in multimodal AI ecosystems, such as social media platforms, where text models generate captions that guide image models, and the resulting images influence the future adaptation of the text model. We take a first step by analyzing such a system, modeling the text model as a multinomial distribution and the image model as a conditional multi-dimensional Gaussian distribution. Our analysis uncovers three key results. First, when one model remains fixed, the other collapses: a frozen image model causes the text model to lose diversity, while a frozen text model leads to an exponential contraction of image diversity, though fidelity remains bounded. Second, in fully interactive systems, mutual reinforcement accelerates collapse, with image contraction amplifying text homogenization and vice versa, leading to a Matthew effect where dominant texts sustain higher image diversity while rarer texts collapse faster. Third, we analyze stabilization strategies implicitly introduced by real-world external influences. Random corpus injections for text models and user-content injections for image models prevent collapse while preserving both diversity and fidelity. Our theoretical findings are further validated through experiments.
\end{abstract}

%%%%%%%%%%%%%%%%%%%%%%%%%%%%%%%%%%%%%%%%%%%%%%%%%%%%
\section{Introduction}

In recent years, the rapid advancement of generative AI has been propelled by significant increases in computing power and the availability of large-scale datasets. These developments have enabled models to produce high-quality text and images, leading to an increasing prevalence of synthetic data online. In fact, it is estimated that for the publishing year 2023, at least over 1\% of all articles were LLM-assisted~\cite{gray2024chatgpt}. Researchers have thus introduced the concept \emph{self-consuming} generative models to describe models trained on their own generated data. Studies have shown that this process can cause models to lose quality and diversity over time~\cite{alemohammad2024self,bohacek2023nepotistically,casco2023toward,dohmatob2024model,dohmatob2024tale,guo2024curious,marchi2024heat,martinez2023combining,martinez2023towards,peterson2025ai,shumailov2024ai,suresh2024rate,wyllie2024fairness}, a phenomenon known as model collapse~\cite{shumailov2024ai} or model autophagy disorder (MAD)~\cite{alemohammad2024self}.

Despite these concerns, less attention has been given to how generative models are no longer just passive consumers of data but are now actively shaping each other's training processes. A clear example of this is found in multimodal generative AI ecosystems, where different models train and generate content based on each other's feedback. For instance, on social media platforms like TikTok, a text model might initially suggest hashtags by sampling from a baseline probability distribution over a fixed set of texts, while an accompanying image model generates visual content associated with each hashtag. As users engage with content featuring these hashtags, the platform scrapes the resulting interaction data and uses it to retrain the text model, recalculating the posterior probabilities of each hashtag based on the success of its associated visuals. Simultaneously, the image model leverages a repository of existing visuals, fine-tuning its parameters on these images to iteratively refine its output distribution. As another example, researchers have demonstrated that two artificial intelligences can communicate purely through linguistic means~\cite{riveland2024natural}. In this study, one AI learned a task and then provided a linguistic description to a ``sister'' AI, which successfully performed the task based on that description.

These instances illustrate how generative models are not just generating data but are actively shaping each other's future learning trajectories through iterative feedback loops. This approach offers practical benefits, such as reducing dependence on human-curated datasets and allowing AI to adapt more quickly to new environments. However, little is known about the long-term effects of this training method. As a few motivating questions, one may ask: (\romannumeral 1) \emph{What happens when a single model is trained in isolation, will it exhibit collapse similar to self-consuming loops, or can it inherently maintain diversity?} (\romannumeral 2) \emph{In a closed system where both models are updated, does mutual interaction strike a balance between them, or does it amplify dominant patterns at the expense of emerging trends?} (\romannumeral 3) \emph{And if such collapse occurs, how to stabilize the system to preserve diversity?}

To the best of our knowledge, the multimodal loop has only been studied empirically before in an \emph{inference loop} setting~\cite{conde2024analyzing}, where the authors also raised the question:``\emph{It would be interesting to study the impact of recursive modality changes when different models are used.}'' In this paper, we take a first theoretical step towards understanding \emph{training loop} dynamics by introducing a novel \emph{co-evolving} system, where a text model and an image model evolve together. We analyze its long-term behavior, characterize its convergence dynamics, and propose stabilization strategies.

%%%%%%%%%%%%%%%%%%%%%%%%%%%%%%%%%%%
\subsection{Contributions} Our contributions are threefold. (\romannumeral 1) \emph{We characterize the isolated dynamics of the text and image models by freezing one model and updating another, and establish their collapse.} When the image model is frozen, we prove that the text model diversity decreases monotonically in expectation and nearly always collapses to zero (see~\cref{subsec:trainable_text_model_with_frozen_image_model,thm:text_model_diversity_frozen_image_model}). Conversely, when the text model is frozen, the image model experiences an exponential decay in diversity~(see~\cref{subsec:trainable_image_model_with_frozen_text_model,thm:image_model_covariance_frozen_text_model}). These results generalize prior findings on discrete and Gaussian distributions with a fixed number of samples~\cite{shumailov2024ai}. Furthermore, we demonstrate that the fidelity of the image model remains bounded despite its diversity shrinking to zero (see~\cref{subsec:trainable_image_model_with_frozen_text_model,thm:image_model_mean_frozen_text_model}), which is previously unexplored.
(\romannumeral 2) \emph{We reveal how the mutual reinforcement between the text and image models accelerates collapse compared to when each model is observed in isolation.} Specifically, as the image model collapses, the text model receives sharper feedback, leading to an accelerated loss of diversity that approximates the theoretical upper bound (see~\cref{subsec:exponential_convergence_due_to_image_model_collapse,thm:acceleratd_text_model_collapse}). In turn, the concentration of the text model probability mass on a small set of texts induces a Matthew effect in the image model, wherein images corresponding to dominant texts retain higher diversity, while those linked to rare texts collapse more rapidly (see~\cref{subsec:matthew_effect_due_to_text_model_collapse,thm:matthew_effect}). 
(\romannumeral 3) \emph{We analyze stabilization strategies based on external information injection and prove that they prevent collapse while maintaining bounded fidelity.} First, we investigate corpus injection, where new texts are randomly added to the model to redistribute probability mass. We prove that this prevents text model collapse by ensuring a strictly positive lower bound on text model diversity (see~\cref{subsec:stabilization_of_the_text_model_via_corpus_injection,thm:stabilization_text_model_via_injection}). Second, we examine user-content injection, incorporating images drawn from an external distribution into the image model training process. We show that this mechanism not only maintains image model diversity above a nonzero threshold but also ensures that its fidelity remains bounded over time (see~\cref{subsec:stabilization_of_the_image_model_via_user-content_injection,thm:stabilization_image_diversity_injection,thm:boundedness_image_fidelity_injection}).

%%%%%%%%%%%%%%%%%%%%%%%%%%%%%%%%%%%
\subsection{Notations} Matrices are represented by bold capital letters, e.g., \(\mA\), vectors by bold lowercase letters, e.g., \(\vy\), and scalars by regular letters, e.g., \(t\). For a matrix \(\mA\), we denote its trace by \(\tr(\mA)\), its nuclear norm (i.e., the sum of its singular values) by \(\|\mA\|_*\), and its square root, when positive definite, by \(\mA^{1/2}\). We adopt the Loewner order \(\prec\) or \(\preceq\) for two symmetric matrices \(\mA\) and \(\mB\). Specifically, we write \(\mA\prec\mB\) (resp.\ \(\mA\preceq\mB\)) if and only if \(\mB-\mA\) is positive definite (resp.\ positive semidefinite). We use \(\mathbb{P}\) for the probability operator, \(\E\) for the expectation operator, \(\Var\) for the variance, and \(\Cov\) for the covariance matrix. Within the expectation operator, the expression to the left of \(|\) represents the variable whose expectation is taken, while the variables to the right specify the given conditions under which the expectation is evaluated. We use \(\mathcal{N}\big(\vy;\vmu, \mSigma\big)\) to denote a multi-dimensional Gaussian distribution evaluated at \(\vy\) with mean vector \(\vmu\) and covariance matrix \(\mSigma\). We use \(\liminf\) (resp.\ \(\limsup\)) to denote the limit inferior (resp.\ limit superior), i.e., the greatest lower bound (resp.\ the least upper bound) of the set of limit points of a sequence.

%%%%%%%%%%%%%%%%%%%%%%%%%%%%%%%%%%%%%%%%%%%%%%%%%%%%
\section{Problem Setup: Co-Evolving System and Diagnostic Measures}

In this section, we first introduce the mathematical model for the co-evolving text-image system, which consists of a text model (a multinomial distribution over texts) and an image model (a multi-dimensional Gaussian distribution conditioned on texts) in~\cref{subsec:the_text_model_and_the_image_model}. We then describe the co-evolving text-image training procedure in~\cref{subsec:the_co-evolving_text-image_system} and finally define diagnostic measures for monitoring the system's behavior, with a focus on diversity and fidelity, in~\cref{subsec:diagnostic_measures_of_the_co-evolving_system}.

%%%%%%%%%%%%%%%%%%%%%%%%%%%%%%%%%%%
\subsection{The Text Model and the Image Model}
\label{subsec:the_text_model_and_the_image_model}

We model the text model as a multinomial distribution over a fixed corpus \(\mathcal{X} = \{ x_1, x_2, \dots, x_K \}\) with 
\begin{equation}
p_t(x_i) = \mathbb{P}(\text{text} = x_i), \quad \sum_{i=1}^{K} p_t(x_i) = 1,
\end{equation}
where \(t\) represents the macro time step. This is a natural choice for representing categorical data such as discrete texts or hashtags. At each generation step \(t\), the probability distribution \(p_t(x_i)\) determines how likely a given text \(x_i\) is to be generated. 

Inspired by modern generative models that typically operate in a structured latent space~\cite{dao2023flow,rombach2022high}, often learned using variational autoencoders (VAEs)~\cite{kingma2014auto} that assume a Gaussian prior, we model the image model as a conditional multi-dimensional Gaussian distribution. Specifically, for each text \(x_i\), we assume that the image model generates vectors \(\vy \in \mathbb{R}^d\) from a multi-dimensional Gaussian distribution
\begin{equation}
q_t(\vy|x_i) \coloneqq \mathcal{N}\big(\vy; \vmu_t(x_i), \mSigma_t(x_i)\big).
\end{equation}
Here, \(\vmu_t(x_i)\) represents the mean image output for \(x_i\) and \(\mSigma_t(x_i)\) denotes the covariance matrix, capturing the variability in generated images.

%%%%%%%%%%%%%%%%%%%%%%%%%%%%%%%%%%%
\subsection{The Co-Evolving Text-Image Training Procedure}
\label{subsec:the_co-evolving_text-image_system}

\Cref{alg:co-evolving_generative_models}~formalizes a co-evolving training procedure in which both the text and image models are updated repeatedly. At the beginning of each macro time step \(t\), the current state of the system is given by the text model probability vector \(\vp_t\) and the image model conditional distributions \(q_t(\cdot| x_i)\) for each text \(x_i\). Within a macro time step, the algorithm first updates the text model \(M_t\) times while the image model remains fixed. In each text model update, the algorithm samples \(N\) texts from the current distribution, generates the corresponding images using the fixed image model, and computes the posterior probability given by~\labelcref{eq:alg_posterior_probability}. Then the text model is updated by averaging these posteriors over the \(N\) samples as in~\labelcref{eq:alg_text_update}. This update reinforces those texts that are more likely to have generated the observed images, much like how TikTok reinforces hashtags that result in more engaging content.

After the text model has been updated, the algorithm turns to the image model. Here, the image model is updated \(N_t\) times while sampling texts from the newly updated text model. For each text \(x_i\), the algorithm computes the sample mean and the sample covariance through~\labelcref{eq:alg_sample_mean,eq:alg_sample_covariance}, respectively. These statistics are then used to update the image model for each text as in~\labelcref{eq:alg_image_update}. This update reflects the image model's adaptation to the repository of existing images, refining its output distribution.

\begin{algorithm}
\caption{Co-Evolving Generative Model Training Procedure}
\label{alg:co-evolving_generative_models}
\begin{algorithmic}[1]
\Require A text model with initial probability vector \(\vp_0\), an image model with initial conditional distributions \(\{q_0(\cdot|x_i)\colon 1\leq i\leq K\}\), and the number of macro time steps \(T\)
\Ensure A trained text model with probability vector \(\vp_T\), a trained image model with conditional distributions \(\{q_T(\cdot|x_i)\colon 1\leq i\leq K\}\)
\For{macro time step \(t = 1\) to \(T\)}
\State Initialize \(\vp_{\text{curr}} \gets \vp_{t-1}\)
\For{\(m = 1\) to \(M_t\)} \Comment{Text model update (repeated \(M_t\) times)}
\State Sample \(N\) texts \(x^{(j)} \sim \vp_{\text{curr}}\) and generate corresponding images \(\vy^{(j)} \sim q_{t-1}(\vy | x^{(j)})\)
\State Compute posterior probabilities:
\begin{equation}
\label{eq:alg_posterior_probability}
p_{\text{curr}}(x_i | \vy^{(j)}) = \frac{p_{\text{curr}}(x_i) q_{t-1}(\vy^{(j)} | x_i)}{\sum_{k=1}^{K} p_{\text{curr}}(x_k)   q_{t-1}(\vy^{(j)} | x_k)}
\end{equation}
\State Update text model:
\begin{equation}
\label{eq:alg_text_update}
p_{\text{curr}}(x_i) \gets \frac{1}{N} \sum_{j=1}^{N} p_{\text{curr}}(x_i | \vy^{(j)})
\end{equation}
\EndFor
\State Set \(\vp_{t} \gets \vp_{\text{curr}}\)
\State Initialize \(q_{\text{curr}}(\cdot | x_i) \gets q_t(\cdot | x_i)\) for each \(x_i\)
\For{\(n = 1\) to \(N_t\)} \Comment{Image model update (repeated \(N_t\) times)}
\State Sample \(N\) texts \(x^{(j)} \sim \vp_t\) and generate corresponding images \(\vy^{(j)} \sim q_{\text{curr}}(\vy | x^{(j)})\)
\For{each \(x_i \in \mathcal{X}\)}
\State Let \(N_i \gets \#\{j\colon x^{(j)} = x_i\}\)
\State Compute sample mean:
\begin{equation}
\label{eq:alg_sample_mean}
\vmu_{\text{curr}}(x_i) = \frac{1}{N_i} \sum_{j:x^{(j)} = x_i} \vy^{(j)}
\end{equation}
\State Compute sample covariance:
\begin{equation}
\label{eq:alg_sample_covariance}
\mSigma_{\text{curr}}(x_i) = \frac{1}{N_i - 1} \sum_{j:x^{(j)} = x_i} \big(\vy^{(j)} - \vmu_{\text{curr}}(x_i)\big)\big(\vy^{(j)} - \vmu_{\text{curr}}(x_i)\big)^\top
\end{equation}
\State Update the image model for \(x_i\):
\begin{equation}
\label{eq:alg_image_update}
q_{\text{curr}}(\cdot | x_i) \gets \mathcal{N}\big(\cdot; \vmu_{\text{curr}}(x_i), \mSigma_{\text{curr}}(x_i)\big)
\end{equation}
\EndFor
\EndFor
\State Set \(q_{t}(\cdot | x_i) \gets q_{\text{curr}}(\cdot | x_i)\) for each \(x_i\)
\EndFor
\end{algorithmic}
\end{algorithm}

%%%%%%%%%%%%%%%%%%%%%%%%%%%%%%%%%%%
\subsection{Diagnostic Measures of the Co-Evolving System}
\label{subsec:diagnostic_measures_of_the_co-evolving_system}

In order to monitor the behavior of our co-evolving system, we quantify both the diversity and fidelity of the models involved. For the text model, diversity determines whether the system is exploring a wide range of texts or collapsing to a few dominant ones. For the image model, two complementary aspects are critical: (\romannumeral 1) its diversity, which captures the variability of generated images; and (\romannumeral 2) its fidelity, which measures how closely the generated images adhere to a desired reference (for example, an initial or canonical distribution). In the sections that follow, we rigorously define these measures.

%%%%%%%%%%%%%%%%%%
\subsubsection{Diversity of the text model} Drawing inspiration from \emph{purity} in quantum information theory~\cite{jaeger2007quantum}, we define the diversity of the text model at macro time step \(t\) as
\begin{equation}
\label{eq:text_model_diversity}
H_t(\vp_t) \coloneqq \sum_{i=1}^{K} \big(p(x_i) - p(x_i)^2\big) = 1-\sum_{i=1}^{K}p(x_i)^2,
\end{equation}
where \(\vp_t = (p_t(x_1), p_t(x_2), \dotsc, p_t(x_K))\) represents the probability distribution of the text model over the corpus. If \(\vp_t\) is a \emph{one-hot vector}, meaning the text model concentrates all probability on a single text (i.e., it always generates the same text), then \(H_t(\vp_t) = 0\). This represents complete mode collapse with no diversity in text generation. If \(\vp_t\) is \emph{uniform} (i.e., \(p_t(x_i) = 1/K\) for all \(i\)), then \(H_t(\vp_t) = 1 - 1/K\), which is the maximum achievable diversity. In this case, all texts are equally likely, indicating a fully exploratory model. Intermediate values of \(H_t(\vp_t)\) indicate partial diversity, where the model has some degree of preference for certain texts while still allowing variation.

%%%%%%%%%%%%%%%%%%
\subsubsection{Fidelity of the image model} The fidelity of the image model reflects how closely the generated images adhere to the intended average style associated with the text \(x_i\). A common approach is to compare the current output with a reference. In our setting, we take the initial mean \(\vmu_0(x_i)\) as the reference. Then, the fidelity of the image model at macro time step \(t\) is measured by the Euclidean distance of the current mean from the reference as
\begin{equation}
\label{eq:image_model_fidelity}
F_t(x_i)\coloneqq\|\vmu_t(x_i)-\vmu_0(x_i)\|_2.
\end{equation}
A small value of \(F_t(x_i)\) indicates that the image model preserves the original style, while a large deviation suggests that the model has drifted away from its initial style.

%%%%%%%%%%%%%%%%%%
\subsubsection{Diversity of the image model} For text \(x_i\), we define the image model diversity as
\begin{equation}
\label{eq:image_model_diversity}
D_t(x_i)\coloneqq\tr\big(\mSigma_t(x_i)^{1/2}\big)=\|\mSigma_t(x_i)\|_*,
\end{equation}
which is expressed in the same units as the image space\footnote{\label{fnote:matrix_square_root} Please refer to~\cref{sec:math_background} for the definition and computation of the matrix square root.}. This measure reflects the typical spread of generated images, unlike \(\tr\big(\mSigma_t(x_i)\big)\) (in squared units) that may overemphasize outlier variances. Intuitively, a large \(D_t(x_i)\) indicates a wide variety of outputs, whereas \(D_t(x_i)\approx 0\) signals mode collapse.

%%%%%%%%%%%%%%%%%%%%%%%%%%%%%%%%%%%%%%%%%%%%%%%%%%%%
\section{Isolating the Dynamics by Freezing One Model}
\label{sec:isolating_the_dynamics_by_freezing_one_model}

We begin by examining a simplified scenario in which only one model is updated while the other remains fixed. In terms of~\cref{alg:co-evolving_generative_models}, this corresponds to the extreme cases of either setting \(M_t=0\) (no text updates) or setting \(N_t=0\) (no image updates). Analyzing these limiting cases allows us to isolate the effects of individual updates and gain insight into their respective roles in the training procedure.

%%%%%%%%%%%%%%%%%%%%%%%%%%%%%%%%%%%
\subsection{Trainable Text Model with Frozen Image Model}
\label{subsec:trainable_text_model_with_frozen_image_model}

We begin our analysis by considering a scenario in which the image model is frozen, i.e., its parameters remain fixed, while the text model continues to update. Although the image model does not change, we still draw random samples from it, so that the only source of randomness in this setting is the variability in the image generation process. We first analyze the expected change in the text model diversity over time in~\cref{thm:text_model_diversity_frozen_image_model}, which establishes a recursion for the diversity measure and characterizes its limiting behavior.

\begin{restatable}[Recursion of text model diversity under frozen image model]{theorem}{TextModelDiversityFrozenImageModel}
\label{thm:text_model_diversity_frozen_image_model}
In~\cref{alg:co-evolving_generative_models}, assume that the image model is frozen, meaning its conditional distributions \( q(\vy | x_i) \) remain fixed (hence, we omit the time subscript \(t\) in \(q_t(\vy|x_i)\)), and that the text model is updated. Then, in expectation, the diversity measure \(H_t\) given by~\labelcref{eq:text_model_diversity} satisfies the recursion
\begin{equation}
\E[H_{t+1}(\vp_{t+1})|\vp_t] = \Big(1-\dfrac{1}{N}\Big)\cdot H_t(\vp_t) + \frac{1}{N}\cdot\Big(1-\sum_{i=1}^{K} \E_{\vy\sim r_t}[Z_i(\vy)^2]\Big),
\end{equation}
where 
\begin{equation}
Z_i(\vy) = \dfrac{p_t(x_i) q(\vy | x_i)}{\sum_{k=1}^{K} p_t(x_k) q(\vy | x_k)}
\end{equation}
is the posterior probability, \(N\) is the number of generated images in each update, and \(r_t(\vy) = \sum_{k=1}^{K} p_t(x_k) q(\vy | x_k)\). As a consequence, the diversity measure is non-increasing in expectation, i.e.,
\begin{equation}
\E[H_{t+1}(\vp_{t+1})] \leq \E[H_t(\vp_t)].
\end{equation}
Moreover, as \(t \to \infty\), the diversity measure stabilizes at a limiting value, which is either (\romannumeral 1) \(0\), which corresponds to \(\vp_\infty\) being a one-hot vector; or (\romannumeral 2) \(H_0(\vp_0)\), which corresponds to \(q(\vy|x_i)\) being identical for all \(1\leq i\leq K\).
\end{restatable}

\ifbool{sketch}{
\begin{proof}[Sketch of proof]
The proof proceeds in several key steps:
\begin{enumerate}
\item We first show that \(\E_{\vy\sim r_t}[Z_i(\vy) | \vp_t]=p_t(x_i)\).
\item We then decompose the second moment as \(\E_{\vy\sim r_t}[Z_i(\vy)^2|\vp_t]=p_t(x_i)^2+\Var_{\vy\sim r_t}(Z_i(\vy)|\vp_t)\).
\item Using the independence of the \(N\) samples in each update, we derive a recursion for the diversity measure \(\E[H_{t+1}(\vp_{t+1})|\vp_t]\) in terms of \(H_t(\vp_t)\).
\item We then establish that \(\E[H_{t}(\vp_{t})]\) is monotone nonincreasing.
\item Finally, by applying the Cauchy--Schwarz inequality, we identify the equality condition under which the diversity reaches its stationary point.
\end{enumerate}
\end{proof}
}

The proof of~\cref{thm:text_model_diversity_frozen_image_model}, along with all other proofs, is deferred to~\cref{sec:proofs_to_theorems}. \Cref{thm:text_model_diversity_frozen_image_model} implies that, unless the image model is uniform across texts, the text model is destined to converge to a degenerate distribution where diversity vanishes. Moreover, from the proof of~\cref{thm:text_model_diversity_frozen_image_model}, we observe that the sequence of probability vectors \(\{\vp_t\}\) forms a vector-valued martingale\footnote{\label{fnote:martingale_background} Please refer to~\cref{sec:math_background} for the definition of a martingale and the martingale convergence theorem.}. Since each \(\vp_t\) lies in the \(K\)-simplex (which is compact), the martingale convergence theorem~\cite{hall2014martingale} guarantees that \(\vp_t\) converges almost surely to a random vector \(\vp_\infty\), which turns out to be either a one-hot vector (i.e., when the image model is non-uniform) or equal to \(\vp_0\) (i.e., when the image model is uniform).

In the recursion for the text model diversity, since the term \(1-\sum_{i=1}^{K}\E_{\vy\sim r_t}[Z_i(\vy)^2]\) is nonnegative, it immediately follows that the per-update reduction in diversity is lower bounded by a factor of \((1-N^{-1})\). In other words, although the diversity is non-increasing, it cannot decrease arbitrarily fast; its reduction is at most as fast as a exponential decay with ratio \((1-N^{-1})\) (a bound that holds regardless of whether the image model is updated), which we formalize in~\cref{cor:lower_bound_text_model_diversity}. The corresponding empirical results can be found in~\cref{subsubsec:exp_text_model_diversity_under_the_frozen_image_model}.

\begin{corollary}[Text model diversity decays at most exponentially]
\label{cor:lower_bound_text_model_diversity}
Following~\cref{alg:co-evolving_generative_models}, the text model diversity \(H_t\) given by~\labelcref{eq:text_model_diversity} satisfies
\begin{equation}
\E[H_{t+1}(\vp_{t+1})] \geq \Big(1-\frac{1}{N}\Big)\cdot \E[H_t(\vp_t)].
\end{equation}
By iterating the inequality we obtain
\begin{equation}
\E[H_t(\vp_t)] \geq \Big(1-\frac{1}{N}\Big)^t\cdot H_0.
\end{equation}
That is, the expected diversity decays at most exponentially with rate \((1-N^{-1})\).
\end{corollary}

%%%%%%%%%%%%%%%%%%%%%%%%%%%%%%%%%%%
\subsection{Trainable Image Model with Frozen Text Model}
\label{subsec:trainable_image_model_with_frozen_text_model}

We then consider the scenario in which the text model is frozen, i.e.,  the probabilities associated with the different texts remain fixed, while the image model continues to update. In~\cref{thm:image_model_covariance_frozen_text_model}, we show that the image model diversity converges to zero at an exponential rate.

\begin{restatable}[Image model diversity decays under frozen text model]{theorem}{ImageModelCovarianceFrozenTextModel}
\label{thm:image_model_covariance_frozen_text_model}
Let the fixed probabilities of the text distribution be \(\vp = (p_1, p_2, \dotsc, p_K)>\vzero\) (where we drop the time subscript \(t\) in \(\vp_t\) for brevity). Then, there exist constants \(C>0\) and \(0<\rho<1\) such that the image model diversity \(D_t(x_i)\) given by~\labelcref{eq:image_model_diversity} satisfies
\begin{equation}
\E[D_t(x_i)] = \E\big[\tr\big(\mSigma_t(x_i)^{1/2}\big)\big]\leq C\rho^t,
\end{equation}
i.e., it converges to zero at an exponential rate.
\end{restatable}

\ifbool{sketch}{
\begin{proof}[Sketch of proof]
The proof proceeds in several key steps:
\begin{enumerate}
\item We first express the updated covariance matrix in terms of the current covariance matrix:
\[
\mSigma_{t+1}(x_i) = \mSigma_t(x_i)^{1/2}\cdot \frac{\mW}{N_i-1}\cdot \mSigma_t(x_i)^{1/2},
\]
where \(\mW\sim\text{Wishart}_d(\mI, N_i-1)\) is a \(d\)-dimensional Wishart random matrix\footnote{\label{fnote:wishart_background} Please refer to~\cref{sec:math_background} for the definition of the Wishart random matrix.} with \((N_i-1)\) degrees of freedom and identity scale matrix.
\item By applying Jensen's operator inequality to the matrix square-root function, we obtain
\[
\E\Big[\Big(\frac{\mW}{N_i-1}\Big)^{1/2}\Big] \prec \mI.
\]
\item Finally, we establish the desired convergence using the submultiplicative property of singular values (i.e., for two diagonalizable matrices \(\mA\) and \(\mB\), we have \(\sigma_j(\mA\mB) \leq \sigma_1(\mA)\sigma_j(\mB)\) for all \(j\), where \(\sigma_j(\cdot)\) denotes the \(j\)th largest singular value).
\end{enumerate}
\end{proof}
}

We remark that this theorem may be viewed as a generalization of~\cite[Proposition 1]{bertrand2024stability}, which assumed a fixed \(N_i\) and potentially relied on the commutativity of \(\mSigma_t(x_i)^{1/2}\) and \(\mW^{1/2}\) (where \(\mW\) denotes a Wishart random matrix\footnote{\label{fnote:wishart_distribution} Please refer to~\cref{sec:math_background} for the definition of the Wishart distribution.}~\cite{livan2018introduction}) in its proof. Empirical results supporting~\cref{thm:image_model_covariance_frozen_text_model} are detailed in~\cref{subsubsec:exp_image_model_diversity_under_the_frozen_text_model}. The exponential collapse of the image model's diversity implies that over time the generated images become increasingly concentrated around their current sample mean. However, while a shrinking covariance suggests that the images are more similar, it does not guarantee that this common mean remains faithful to the reference distribution. This observation naturally raises the question: How does the final fidelity of the image model behave? Specifically, does the collapse in diversity ensure that the images are of high quality (i.e., that the sample mean is close to the reference mean), or can the mean itself drift away, thereby compromising fidelity? We formalize this intuition by deriving an upper bound on the fidelity of the image model in~\cref{thm:image_model_mean_frozen_text_model}. This bound depends on the parameters \(C\), \(\rho\), \(N\), and \(p_i\), and it quantifies the interplay between the collapse in diversity and the stability of the mean.

\begin{restatable}[Boundedness of image model fidelity under the frozen text model]{theorem}{ImageModelMeanFrozenTextModel}
\label{thm:image_model_mean_frozen_text_model}
Let the fixed probabilities of the text distribution be \(\vp = (p_1, p_2, \dotsc, p_K)>\vzero\) (where we drop the time subscript \(t\) in \(\vp_t\) for brevity). Suppose there exist constants \(C>0\) and \(0<\rho<1\) such that the image model diversity given by~\labelcref{eq:image_model_diversity} satisfies
\begin{equation}
\E[D_t(x_i)] = \E\big[\tr\big(\mSigma_t(x_i)^{1/2}\big)\big]\leq C\rho^t.
\end{equation}
Then, the image model fidelity defined by~\labelcref{eq:image_model_fidelity}, i.e., the expected deviation of the mean vectors satisfies
\begin{equation}
\E[F_t(x_i)] = \E[\|\vmu_{\infty}(x_i) - \vmu_0(x_i)\|_2]\leq\dfrac{\sqrt{2}C}{\sqrt{(N+1)p_i}\cdot(1-\rho)}.
\end{equation}
\end{restatable}

\ifbool{sketch}{
\begin{proof}[Sketch of proof]
The proof proceeds in several key steps:
\begin{enumerate}
\item We first show that, conditioned on \(\vmu_t(x_i)\) and \(N_i\), \(\vmu_{t+1}(x_i)-\vmu_t(x_i)\sim\mathcal{N}(\vzero, \mSigma_t(x_i)/N_i)\), and we derive the corresponding conditional expectation for \(\|\vmu_{t+1}(x_i)-\vmu_t(x_i)\|_2^2\).
\item Using the law of total expectation over the binomially distributed \(N_i\sim\text{Binomial}(N, p_i)\), we obtain an upper bound for the expected squared deviation.
\item Next, by applying the Cauchy--Schwarz inequality, we upper bound \(\E[\|\vmu_{t+1}(x_i)-\vmu_t(x_i)\|_2]\).
\item Finally, using the triangle inequality over the successive updates, we derive an overall upper bound on \(\E[\|\vmu_{\infty}(x_i)-\vmu_0(x_i)\|_2]\).
\end{enumerate}
\end{proof}
}

\Cref{thm:image_model_mean_frozen_text_model} provides an upper bound on the image model fidelity. This bound is inversely proportional to \(\sqrt{(N+1)p_i}\). Thus, as the total number of samples \(N\) increases or the probability \(p_i\) of the text \(x_i\) is higher, meaning the text \(x_i\) is sampled more frequently, the fidelity measure becomes smaller. The term \((1-\rho)^{-1}\) indicates that if the covariance matrices decay rapidly (i.e., \(\rho\) is small), then the bound becomes tighter. Conversely, if \(\rho\) is close to \(1\) (slow decay), the fidelity bound is looser, implying a greater potential drift.

%%%%%%%%%%%%%%%%%%%%%%%%%%%%%%%%%%%%%%%%%%%%%%%%%%%%
\section{Image-Driven Acceleration in Text Model Collapse}
\label{sec:image-driven_acceleration_in_text_model_collapse}

In this section, we investigate the dynamics of text model collapse as mediated by the behavior of the image model. Recall that in \cref{subsec:trainable_text_model_with_frozen_image_model} we analyzed a simplified setting where the image model is frozen. Under that scenario, the loss of diversity in the text model is governed solely by the inherent variability of the image generation process. As shown in~\cref{subsec:slow_convergence_due_to_large_covariance_image_model}, a carefully constructed frozen image model can yield an arbitrarily slow collapse of the text model diversity. In contrast, in~\cref{subsec:exponential_convergence_due_to_image_model_collapse} we consider the scenario where the image model is allowed to update sufficiently (\(N_t\gg 1\) in~\cref{alg:co-evolving_generative_models}). In this latter case, the image model rapidly contracts its covariance, thereby sharpening the feedback provided to the text model. Consequently, the text model collapse is accelerated and its dynamics more closely approach the theoretical bound derived earlier. This comparison demonstrates that while a frozen image model leads to a gradual loss of diversity in the text model, an actively updating (and collapsing) image model amplifies this effect, driving the text model to collapse more quickly.

%%%%%%%%%%%%%%%%%%%%%%%%%%%%%%%%%%%
\subsection{Slow Convergence Due to Large Covariance Image Model}
\label{subsec:slow_convergence_due_to_large_covariance_image_model}

According to~\cref{thm:text_model_diversity_frozen_image_model}, the update of the text model diversity is governed by
\begin{equation}
1 - \sum_{i=1}^{K}\E_{\vy\sim r_t}[Z_i(\vy)^2], \quad
\text{where } Z_i(\vy) = \frac{p_t(x_i)q(\vy|x_i)}{\sum_{k=1}^{K} p_t(x_k)q(\vy|x_k)}.
\end{equation}
Intuitively, if the posterior distribution is not highly concentrated, then the per-update reduction in diversity will be very small. In~\cref{thm:arbitrarily_small_convergence_large_covariance}, we formalize this intuition by considering the case where the conditional distributions of the image model are Gaussians with large covariance matrices.

\begin{restatable}[Arbitrarily slow convergence under large covariances]{theorem}{ArbitrarilySmallConvergenceLargeCovariance}\
\label{thm:arbitrarily_small_convergence_large_covariance}
For any \(\varepsilon>0\) and any \(t\) with \(H_t(\vp_t)>0\), there exists a family of conditional distributions \(q(\vy|x_i)\) (for example, with covariance matrices equal to \(\sigma^2 \mI\) and with \(\sigma\) chosen sufficiently large) such that the recursion in~\cref{thm:text_model_diversity_frozen_image_model} satisfies
\begin{equation}
\Delta_t \coloneqq H_t(\vp_t) - \E[H_{t+1}(\vp_{t+1})|\vp_t] < \varepsilon H_t(\vp_t).
\end{equation}
In other words, by choosing the image model to be sufficiently diffuse, the per-update reduction in diversity can be made arbitrarily small relative to the current diversity. Hence, the overall convergence can be made arbitrarily slow.
\end{restatable}

\ifbool{sketch}{
\begin{proof}[Sketch of proof]
The proof proceeds in several key steps:
\begin{enumerate}
\item First, we express \(\Delta_t\) in terms of \(H_t(\vp_t)\) and \(\sum_{i=1}^K \E_{\vy\sim r_t}[Z_i(\vy)^2]\).
\item Next, we choose the conditional distributions to be Gaussian, i.e., \(q(\vy|x_i) = \mathcal{N}(\vy; \vmu_i, \sigma^2 \mI)\), and show that as \(\sigma \to \infty\), \(Z_i(\vy) \to p_t(x_i)\) for almost every \(\vy\).
\item By performing the change of variable \(\vy \mapsto \sigma\vy\) in the expectation and applying the dominated convergence theorem, we deduce that \(\E_{\vy\sim r_t}[Z_i(\vy)^2] \to p_t(x_i)^2\).
\item Therefore, for any \(\varepsilon>0\), one can choose \(\sigma\) sufficiently large such that \(\Delta_t < \varepsilon H_t(\vp_t)\).
\end{enumerate}
\end{proof}
}

\Cref{thm:arbitrarily_small_convergence_large_covariance} shows that the convergence rate of the text model can be made arbitrarily slow by appropriately selecting the image model. We remark that, due to the scaling properties of the Gaussian distribution, a scenario with large covariances is equivalent to one where the covariances are fixed while the mean vectors are clustered very closely together. In either case, the differences between the image outputs corresponding to different texts become negligible, leading to nearly uniform posterior probabilities and, consequently, to a slower decay in text model diversity. The experimental findings validating~\cref{thm:arbitrarily_small_convergence_large_covariance} are presented in~\cref{subsubsec:exp_text_model_diversity_under_the_frozen_image_model}.

%%%%%%%%%%%%%%%%%%%%%%%%%%%%%%%%%%%
\subsection{Exponential Convergence Due to Image Model Collapse}
\label{subsec:exponential_convergence_due_to_image_model_collapse}

In the previous~\cref{subsec:slow_convergence_due_to_large_covariance_image_model}, we examined how the text model collapse can be arbitrarily slow when the image model is frozen. In~\cref{thm:acceleratd_text_model_collapse}, we demonstrate that when the image model is trainable and its covariance matrices shrink (i.e., when the image model collapses), the posterior probabilities become highly concentrated. As a consequence, the rate at which the text model loses diversity approaches the theoretical lower bound established in~\cref{cor:lower_bound_text_model_diversity}. In terms of~\cref{alg:co-evolving_generative_models}, this scenario corresponds to setting \(M_t=1\) and \(N_t\gg 1\) in each macro time step \(t\).

\begin{restatable}[Exponential convergence under the trainable image model]{theorem}{AcceleratdTextModelCollapse}
\label{thm:acceleratd_text_model_collapse}
For any macro time step \(t\), suppose that for every text \(x_i\) and for every inner image-model update step \(s\) (with \(s=0\) corresponding to the start of macro time step \(t\)), there exists a constant \(0<\rho<1\) such that
\begin{equation}
\E\big[\tr\big(\mSigma_{t,s}(x_i)^{1/2}\big) \big|p_t(x_i), \mSigma_{t,0}(x_i)\big] \leq \tr\big(\mSigma_{t,0}(x_i)^{1/2}\big)\cdot \rho^s.
\end{equation}
Suppose further that there exists a constant \(\Gamma > 0\) such that at every inner image-model update step \(s\) and for any two distinct texts \(x_i\) and \(x_j\) with \(p_t(x_i) > 0\) and \(p_t(x_j) > 0\),
\begin{equation}
\|\vmu_{t, s}(x_i) - \vmu_{t, s}(x_j)\|_2 \geq \Gamma.
\end{equation}
Then, for any \(\varepsilon>0\), if the number \(N_t\) of image-model (inner) updates in macro time step \(t\) is sufficiently large, the text-model update satisfies
\begin{equation}
\Delta_t \coloneqq H_t(\vp_t) - \E\big[H_{t+1}(\vp_{t+1}) \big|\{p_t(x_k), \mSigma_{t,0}(x_k)\colon 1\le k\le K\}\big] > \dfrac{1-\varepsilon}{N}\cdot H_t(\vp_t).
\end{equation}
\end{restatable}

\ifbool{sketch}{
\begin{proof}[Sketch of proof]
The proof proceeds in several key steps:
\begin{enumerate}
\item First, we express \(\Delta_t\) in terms of \(\sum_{i=1}^K \E_{\vy\sim r_t}[Z_i(\vy)^2]\) and \(\sum_{i=1}^Kp_t(x_i)^2\). We aim to lower-bound \(\E_{\vy\sim r_t}[Z_i(\vy)^2]\).
\item Next, we show that if \(N_t\) is large, then the corresponding Gaussian densities become highly concentrated. In particular, for any two distinct texts \(x_i\) and \(x_j\) (with \(p_t(x_i), p_t(x_j)>0\)), there exist constants \(\varepsilon'>0\) and \(r_0>0\) such that for all \(\vy\) within the Mahalanobis ball \(\mathcal{B}_M(\vmu_t(x_i), r_0)\), we have \(q_t(\vy\mid x_j) \le \varepsilon' q_t(\vy\mid x_i)\).
\item This inequality implies that, on a high-probability set, the posterior probabilities \(Z_i(\vy)\) are sharply concentrated around \(p_t(x_i)\). Consequently, we can lower-bound \(\E_{\vy\sim r_t}[Z_i(\vy)^2]\) by an expression that is close to \(p_t(x_i)^2\), thereby ensuring that the reduction \(\Delta_t\) is bounded below by approximately \((1-\varepsilon)\cdot H_t(\vp_t)/N\).
\end{enumerate}
\end{proof}
}

The proof of~\cref{thm:acceleratd_text_model_collapse} is provided in~\cref{sec:proofs_to_theorems}. We now discuss the assumptions and implications of it. First, the mean separation condition is modest: by combining the result of~\cref{thm:image_model_mean_frozen_text_model} with the requirement that the initial means, \(\vmu_0(x_i)\)'s, are sufficiently separated, we ensure that distinct texts induce adequately different image outputs. This separation is both theoretically sound and practically achievable, as many applications are designed with well-differentiated initial embeddings. Second, the assumption that \(N_t \gg 1\) means that the co-evolving system leverages extensive image training to generate more accurate and confident responses, which in turn accelerates the collapse of the text model. Overall, this image-driven acceleration in text model collapse effectively pushes the text model towards its theoretical bound to exponential convergence of rate close to \((1-N^{-1})\), as characterized in~\cref{cor:lower_bound_text_model_diversity}. Empirical results supporting~\cref{thm:acceleratd_text_model_collapse} are presented in~\cref{subsec:exp_image_driven_acceleration_in_text_model_collapse}.

%%%%%%%%%%%%%%%%%%%%%%%%%%%%%%%%%%%%%%%%%%%%%%%%%%%%
\section{Text-Driven Matthew Effect in Image Model Collapse}
\label{sec:text-driven_matthew_effect_in_image_model_collapse}

In this section, we explore how the dynamics of the text model induce a Matthew effect in the collapse of the image model. We begin by providing a fine-grained analysis in~\cref{subsec:differential_convergence_rate_of_the_image_model}, which demonstrates that the diversity of image models converges at different rates across texts. Specifically, image models associated with high-probability (dominant) texts exhibit a slower decay in diversity compared to those linked to low-probability (rare) texts. Then, in~\cref{subsec:matthew_effect_due_to_text_model_collapse}, we show that this differential convergence leads to a self-reinforcing feedback loop: as the text model collapses and concentrates probability mass on a few dominant texts, the image models corresponding to these texts retain more diversity, while those linked to rare texts degrade more rapidly. This process exemplifies the \emph{Matthew effect}, a phenomenon where initial advantages compound over time, which is commonly phrased as ``the rich get richer and the poor get poorer.'' In our setting, dominant texts not only become more probable but also sustain higher-diversity image generations, whereas rare texts vanish from both the text and image distributions. Together, we demonstrate that the text model actively shapes the long-term stability of the image model, reinforcing the prominence of popular texts within the co-evolving system.

%%%%%%%%%%%%%%%%%%%%%%%%%%%%%%%%%%%
\subsection{Differential Convergence Rate of the Image Model}
\label{subsec:differential_convergence_rate_of_the_image_model}

In this subsection, we quantitatively characterize the convergence behavior of the image model and, importantly, show that its collapse rate is not uniform but varies across different texts. Building on~\cref{thm:image_model_covariance_frozen_text_model}, by deriving an explicit approximation for the convergence rate \(\rho\) using a Taylor expansion of the matrix square-root function and properties of the Wishart distribution in~\cref{thm:diff_convergence_rate_image_model}, we demonstrate that more frequently sampled texts tends to collapse at a slower speed compared to those for less frequent texts.

\begin{restatable}[Differential convergence rate of the image model]{theorem}{DiffConvergenceRateImageModel}
\label{thm:diff_convergence_rate_image_model}
Assume that \(N\gg d\). For a fixed macro time step \(t\), let the probabilities of the text distribution be \(\vp=(p_1, p_2, \dotsc, p_K)>\vzero\). Then the convergence rate \(\rho\) derived in~\cref{thm:image_model_covariance_frozen_text_model} is approximately
\begin{equation}
1-\dfrac{d+1}{8(N+1)p_i}.
\end{equation}
\end{restatable}

\ifbool{sketch}{
\begin{proof}[Sketch of proof]
The proof proceeds in several key steps:
\begin{enumerate}
\item We first recast the problem as conditioning on \(N_i\) and estimating the expected square root of \(\mW\sim\text{Wishart}_d(\mI, N_i-1)\).
\item Next, we perform a Taylor expansion of the matrix square root \(\mW^{1/2}\) about its mean \((N_i-1)\mI_d\) up to second order.
\item We then carefully bound the remainder using standard results on the Fr\'echet derivative of the square-root function.
\item Finally, by taking the expectation over \(N_i\) (using the law of total expectation) and noting that the mean value of \(N_i\) is approximately \((N+1)p_i\) when \(N\gg d\), we derive the approximate convergence rate.
\end{enumerate}
\end{proof}
}

See~\cref{subsubsec:exp_image_model_diversity_under_the_frozen_text_model} for the empirical results corresponding to~\cref{thm:diff_convergence_rate_image_model}. We remark that \cref{thm:diff_convergence_rate_image_model} has several important implications. First, the term \((d+1)\) in the numerator indicates that higher-dimensional image representations (i.e., larger \(d\)) lead to a faster collapse of diversity. Second, the presence of \((N+1)\) in the denominator shows that increasing the number of samples per update (i.e., larger \(N\)) slows down the collapse, suggesting that larger batch sizes can help preserve image diversity. Finally, the dependence on \(p_i\) reveals that image models corresponding to more frequently sampled texts (i.e., higher \(p_i\)) collapse more slowly, whereas those associated with rarer texts lose diversity more rapidly, potentially reinforcing the dominance of popular content. We remark that \Cref{thm:diff_convergence_rate_image_model} extends existing results: (\romannumeral 1) in contrast to \cite[Theorem~3]{suresh2024rate}, which derived the convergence rate for a single one-dimensional Gaussian, our result generalizes the analysis to higher dimensions and to systems with multiple Gaussians; and (\romannumeral 2) compared with \cite[Proposition~1]{bertrand2024stability}, we provide an explicit convergence rate that is dependent on \(d\), \(N\), and \(p_i\).

%%%%%%%%%%%%%%%%%%%%%%%%%%%%%%%%%%%
\subsection{Matthew Effect Due to Text Model Collapse}
\label{subsec:matthew_effect_due_to_text_model_collapse}

In this section, we examine how the text model collapse not only suppresses diversity in the text domain but also accelerates the collapse of the image models corresponding to less frequent texts, which is a clear instance of the Matthew effect. We formalize this effect in~\cref{thm:matthew_effect}, which quantifies how differential convergence rates in the image models are magnified as the text model diversity diminishes.

\begin{restatable}[Matthew effect of image model diversity under text model collapse]{theorem}{MatthewEffect}
\label{thm:matthew_effect}
Assume that the text model diversity given by~\labelcref{eq:text_model_diversity} satisfies \(H_t(\vp_t)=\varepsilon\), and that the image model convergence rates \(\{\rho(x_i)\colon 1\leq i\leq K\}\) are given by~\cref{thm:diff_convergence_rate_image_model}. Without loss of generality, denote the dominant and the rarest texts by
\begin{equation}
x_1 = \arg\max_{1\le i\le K} p_t(x_i),\quad \text{and}\quad x_K = \arg\min_{1\le i\le K} p_t(x_i).
\end{equation}
Then, one has
\begin{equation}
\frac{\rho(x_1)}{\rho(x_K)} \geq \max\Big(\dfrac{(d+1)(K-1)}{8(N+1)}\cdot\varepsilon^{-1}, 1\Big).
\end{equation}
\end{restatable}

\ifbool{sketch}{
\begin{proof}[Sketch of proof]
The proof proceeds in several key steps:
\begin{enumerate}
\item We first recast the problem of finding a lower bound on \(\rho(x_1)/\rho(x_K)\) into the following optimization problem:
\begin{equation}
\begin{aligned}
\min\quad & p_t(x_K)^{-1} - p_t(x_1)^{-1},\\
\text{s.t.}\quad & p_t(x_1)^2 + \cdots + p_t(x_K)^2 \leq 1-\varepsilon,\\
& p_t(x_1) + \cdots + p_t(x_K) = 1,\\
& p_t(x_1) \geq \cdots \geq p_t(x_K) > 0.
\end{aligned}
\end{equation}
\item By relaxing the problem and exploiting the convexity of the quadratic term \(p_t(x_2)^2 + \dotsb + p_t(x_K)^2\), we show that the optimum is attained when the lower probabilities are equal. This observation allows us to derive an explicit lower bound on \(p_t(x_K)^{-1} - p_t(x_1)^{-1}\).
\item Finally, substituting the solution of the relaxed optimization problem into the expression for \(\rho(x_i)\) yields the desired inequality.
\end{enumerate}
\end{proof}
}

The proof of~\cref{thm:matthew_effect} is provided in~\cref{sec:proofs_to_theorems}. It demonstrates that as the text model collapses (i.e., exhibits a small diversity \(\varepsilon\)), the disparity between the convergence rates of the image models for dominant versus rare texts is magnified proportional to \(\varepsilon^{-1}\). In this scenario, the image model corresponding to the dominant text maintains its diversity for a significantly longer period, while those associated with infrequent texts collapse rapidly.

%%%%%%%%%%%%%%%%%%%%%%%%%%%%%%%%%%%%%%%%%%%%%%%%%%%%
\section{Stabilization of the Co-Evolving System}
\label{sec:stabilization_of_the_co-evolving_system}

In previous sections, we analyzed a close co-evolving system in which a text model and an image model iteratively reinforce each other (see \Cref{alg:co-evolving_generative_models}). We showed that this training procedure drives the text model to concentrate its probability mass on a single text, which in turn induces a Matthew effect leading to the collapse of the image model. However, real-world systems are continuously influenced by external factors, such as the influx of new topics and user-generated content. In this section, we investigate how these external influences can mitigate collapse. In \cref{subsec:stabilization_of_the_text_model_via_corpus_injection}, we examine stabilization via corpus injection, where new texts are randomly added to the corpus to redistribute probability mass and prevent the text model's collapse. We then explore stabilization via user-content injection in \cref{subsec:stabilization_of_the_image_model_via_user-content_injection}, in which images drawn from a fixed distribution are incorporated into the training of the image model. We demonstrate that these external injections not only prevent the image model diversity from collapsing, but also ensure that the fidelity remains bounded over time.

%%%%%%%%%%%%%%%%%%%%%%%%%%%%%%%%%%%
\subsection{Stabilization of the Text Model via Corpus Injection}
\label{subsec:stabilization_of_the_text_model_via_corpus_injection}

In social media, an influx of new trends and topics continuously emerges, ensuring that the co-evolving system is never truly closed. To mimic this real-world phenomenon, we extend~\cref{alg:co-evolving_generative_models} by randomly injecting new texts into the corpus. The revised training procedure is summarized in~\cref{alg:co-evolving_with_text_injection}. At the beginning of each macro time step \(t\), suppose the current text model is represented by the probability vector \(\vp_t = (p_t(x_1),\dotsc,p_t(x_K))\). With probability \(\alpha\), a new text \(x_{\mathrm{new}}\) is injected into the corpus. The injection is performed by reallocating a small fraction \(\varepsilon>0\) of the total probability mass from the existing texts to the new text. Formally, we update the text model as follows:
\begin{equation}
\begin{cases}
p_t(x_i) \gets (1-\varepsilon)\cdot p_t(x_i),& 1\le i\le K,\\
p_t(x_{\mathrm{new}}) \gets \varepsilon. &\\
\end{cases}
\end{equation}
In parallel, we initialize the image model for \(x_{\mathrm{new}}\) with a Gaussian distribution:
\begin{equation}
q_t(\vy|x_{\mathrm{new}}) \gets \mathcal{N}\big(\vy; \vmu(x_{\mathrm{new}}), \mSigma(x_{\text{new}})\big),
\end{equation}
where \(\vmu(x_{\mathrm{new}})\) and \(\mSigma(x_{\text{new}})\) are the initialized parameters. After any injection event (or if no injection occurs), the algorithm proceeds with the standard training procedure as described in~\cref{alg:co-evolving_generative_models}. 

A natural question is: Does this random injection of new texts mitigate the text model collapse? We provide an assertive answer in~\cref{thm:stabilization_text_model_via_injection}, which shows that with a fixed injection probability and fraction, the external influx of new texts ensures that the text model's diversity retains a strictly positive lower bound, preventing collapse into a degenerate distribution. This confirms that random injections effectively counterbalance the self-reinforcing dynamics of a closed system, sustaining diversity in real-world settings.

\begin{algorithm}[t]
\caption{Co-Evolving Generative Model Training Procedure with Text Injection}
\label{alg:co-evolving_with_text_injection}
\begin{algorithmic}[1]
\Require A corpus \(\mathcal{X}\) with \(K\) texts. A text model with initial probability vector \(\vp_0\), an image model with initial conditional distributions \(\{q_0(\cdot|x_i)\colon 1\leq i\leq K\}\), the number of macro time steps \(T\), text injection probability \(\alpha>0\), text injection fraction \(\varepsilon>0\)
\For{macro time step \(t=1\) to \(T\)}
\If{a random event occurs with probability \(\alpha\)}
\Comment{Inject a new text with probability \(\alpha\)}
\For{each \(x_i \in \mathcal{X}\)}
\State Update: \(p_{t-1}(x_i) \gets (1-\varepsilon)\cdot p_{t-1}(x_i)\)
\EndFor
\State Inject a new text \(x_{\text{new}}\) to the corpus \(\mathcal{X}\) with 
\begin{equation}
p_t(x_{\text{new}}) \gets \varepsilon
\end{equation}
\State Initialize its image model:
\begin{equation}
q_t(\vy|x_{\text{new}}) \gets \mathcal{N}\big(\vy; \vmu(x_{\text{new}}), \mSigma(x_{\text{new}})\big)
\end{equation}
\EndIf
\State Run the standard training procedure as in~\cref{alg:co-evolving_generative_models} for macro time step \(t\)
\EndFor
\end{algorithmic}
\end{algorithm}

\begin{restatable}[Stabilization of text model diversity under text injection]{theorem}{StabilizationTextModelViaInjection}
\label{thm:stabilization_text_model_via_injection}
Under the co-evolving training procedure with text injection as described in~\cref{alg:co-evolving_with_text_injection}, the text model diversity \(H_t\) defined by~\labelcref{eq:text_model_diversity} is prevented from collapsing. More precisely,
\begin{equation}
\liminf_{t\to\infty}\E[H_t] \geq \frac{2\alpha(1-N^{-1})(\varepsilon - \varepsilon^2)}{1-(1-\alpha)(1-N^{-1})}.
\end{equation}
\end{restatable}

\ifbool{sketch}{
\begin{proof}[Sketch of proof]
The proof proceeds in several key steps:
\begin{enumerate}
\item We first distinguish two cases: when an injection occurs (with probability \(\alpha\)) and when no injection occurs (with probability \((1-\alpha)\)).
\item In the no-injection case, we invoke~\cref{cor:lower_bound_text_model_diversity} to show that the text model diversity contracts by at most a factor of \((1-N^{-1})\) per update.
\item In the injection case, we compute the diversity immediately after injection; specifically, when a fraction \(\varepsilon\) of the probability mass is reallocated to a new text, the diversity increases to at least \(2\varepsilon-2\varepsilon^2\), which is then reduced by the same contraction factor.
\item Finally, applying the law of total expectation and iterating over successive updates yields the desired lower bound.
\end{enumerate}
\end{proof}
}

\Cref{thm:stabilization_text_model_via_injection} provides two key insights. First, by ensuring that new texts are injected with probability \(\alpha\), \cref{alg:co-evolving_with_text_injection} prevents model collapse by interrupting the self-reinforcing feedback that would otherwise drive the text model towards a degenerate distribution. Second, the theorem's quantitative lower bound highlights the influence of the key parameters: the injection probability \(\alpha\) determines how often new texts are introduced, with a higher \(\alpha\) leading to more frequent injections that bolster diversity; the injection fraction \(\varepsilon\) controls the proportion of probability mass reallocated to a new text, where \(\varepsilon\) being closer to \(1/2\) introduces greater diversity; and the sample size \(N\) governs the contraction rate during standard text updates with larger values of \(N\) resulting in larger diversity. For empirical validation of~\cref{thm:stabilization_text_model_via_injection}, see~\cref{subsubsec:exp_stabilization_of_the_text_model_via_corpus_injection}.

%%%%%%%%%%%%%%%%%%%%%%%%%%%%%%%%%%%
\subsection{Stabilization of the Image Model via User-Content Injection}
\label{subsec:stabilization_of_the_image_model_via_user-content_injection}

Aside from the influx of new trends and topics, user-generated content also plays a crucial role by acting as a regularizing force that prevents image model collapse. In this section, we demonstrate that by incorporating a fixed number \(N_0>0\) of images drawn from an external (user-content) distribution during the image model update, the updated image model retains strictly positive diversity and is bounded in fidelity.

To formalize this idea, assume that for each text \(x_i\) the user-content images are drawn i.i.d.\ from a Gaussian distribution
\begin{equation}
q_{\text{user}}(\vy|x_i)=\mathcal{N}\big(\vy;\vmu_{\text{user}}(x_i), \mSigma_{\text{user}}(x_i)\big),
\end{equation}
where the covariance matrix \(\mSigma_{\text{user}}(x_i)\) is non-degenerate. At each macro time step \(t\), in addition to generating \(N_i\) images from the current image model \(q_t(\vy|x_i)\) (with mean vector \(\vmu_t(x_i)\) and covariance matrix \(\mSigma_t(x_i)\)), we also inject \(N_0\) user-content images \(\vu^{(1)},\dotsc,\vu^{(N_0)}\) drawn from \(q_{\text{user}}(\cdot|x_i)\). We then compute the combined sample mean and covariance, which are used to update the image model. The training procedure is summarized in~\cref{alg:co-evolving_with_image_injection}.

\begin{algorithm}[t]
\caption{Co-Evolving Generative Model Training Procedure with Image Injection}
\label{alg:co-evolving_with_image_injection}
\begin{algorithmic}[1]
\Require A corpus \(\mathcal{X}\) with \(K\) texts. A text model with initial probability vector \(\vp_0\), an image model with initial conditional distributions \(\{q_0(\cdot|x_i)\colon 1\leq i\leq K\}\), the number of macro time steps \(T\), user-content injection number \(N_0>0\)
\For{macro time step \(t=1\) to \(T\)}
\State Run the standard training procedure of the text model as in~\cref{alg:co-evolving_generative_models}
\State Sample \(N\) texts \(x^{(j)} \sim \vp_t\)
\State Generate corresponding images \(\vy^{(j)} \sim q_{t-1}(\vy | x^{(j)})\)
\For{each \(x_i \in \mathcal{X}\)}
\State Let \(N_i \gets \#\{j\colon x^{(j)} = x_i\}\)
\State Bring in user-content images
\Comment{Inject user-content images}
\begin{equation}
\vu^{(n)} \sim q_{\text{user}}(\cdot| x_i) = \mathcal{N}\big(\vmu_{\text{user}}(x_i), \mSigma_{\text{user}}(x_i)\big), \quad n=1, \dotsc, N_0
\end{equation}
\State Compute sample mean:
\begin{equation}
\vmu_{t}(x_i) = \frac{1}{N_i+N_0} \Big(\sum_{j:x^{(j)} = x_i} \vy^{(j)} + \sum_{n=1}^{N_0}\vu^{(n)}\Big)
\end{equation}
\State Compute sample covariance:
\begin{equation}
\begin{aligned}
\mSigma_{t}(x_i) = \frac{1}{N_i + N_0 - 1} \Big(&\sum_{j:x^{(j)} = x_i} \big(\vy^{(j)} - \vmu_{t}(x_i)\big)\big(\vy^{(j)} - \vmu_{t}(x_i)\big)^\top\\
&+\sum_{n=1}^{N_0}\big(\vu^{(n)} - \vmu_{t}(x_i)\big)\big(\vu^{(n)} - \vmu_{t}(x_i)\big)^\top\Big)
\end{aligned}
\end{equation}
\State Update the image model for \(x_i\):
\begin{equation}
q_{t}(\cdot | x_i) \gets \mathcal{N}\big(\cdot; \vmu_{t}(x_i), \mSigma_{t}(x_i)\big)
\end{equation}
\EndFor
\EndFor
\end{algorithmic}
\end{algorithm}

One may ask: Does the injection of user-generated images mitigate the collapse of the image model diversity, in a manner analogous to how external text injections preserve text model diversity? The answer is yes. As we demonstrated for the text model, external injections can prevent complete collapse by continuously reintroducing variability into the system. In~\cref{thm:stabilization_image_diversity_injection}, we show that when a fixed number \(N_0>0\) of user-content images drawn from a reference distribution with fixed mean and covariance is injected at each macro time step, the diversity of the image model remains bounded from below by a positive constant.

\begin{restatable}[Stabilization of image model diversity under image injection]{theorem}{StabilizationImageDiversityInjection}
\label{thm:stabilization_image_diversity_injection}
Following the co-evolving training procedure with image injection as described in~\cref{alg:co-evolving_with_image_injection}, the image model diversity \(D_t(x_i)\) given by~\labelcref{eq:image_model_diversity} is prevented from collapsing. More precisely, for all \(t\geq 1\),
\begin{equation}
\E[D_t(x_i)]\geq \alpha\cdot\dfrac{1}{\sqrt{(N_0-1)(N+N_0-1)}}\cdot\E\big[\tr\big(\mSigma_{\text{user}}(x_i)^{1/2}\big)\big],
\end{equation}
where \(\alpha\) is the unique scalar such that \(\E[\mW^{1/2}] = \alpha\mI_d\) with \(\mW\sim\text{Wishart}_d(\mI,N_0-1)\) being a \(d\)-dimensional Wishart matrix with \((N_0-1)\) degrees of freedom and identity scale matrix.
\end{restatable}

\ifbool{sketch}{
\begin{proof}[Sketch of proof]
The proof proceeds in several key steps:
\begin{enumerate}
\item We first write the combined sample covariance after user-content injection as a weighted sum of the sample covariance from the current image model (which is positive semidefinite) and that from the injected user-content images, along with positive semidefinite cross terms.
\item Next, we express the sample covariance from the user-content images in terms of the reference covariance \(\mSigma_{\text{user}}(x_i)\) and a Wishart random matrix \(\mW\sim\text{Wishart}_d(\mI, N_0-1)\).
\item Finally, by applying the law of total expectation, we deduce the stated lower bound.
\end{enumerate}
\end{proof}
}

The implications of~\cref{thm:stabilization_image_diversity_injection} are twofold. First, the coefficient
\begin{equation}
\frac{1}{\sqrt{(N_0-1)(N+N_0-1)}}
\end{equation}
shows that increasing the number \(N_0\) of injected images improves the lower bound on the image model diversity, whereas a larger number \(N\) of generated images without a corresponding increase in \(N_0\) reduces the regularizing effect. Second, although the theorem is stated for a fixed \(N_0\), the analysis generalizes to settings where the number of injected images is random (e.g., Poisson distributed) as long as there is a positive probability of injecting more than one image. In such cases, the effective injection parameter becomes the conditional expectation of the number of injected images given that at least two are injected.

One may inquire whether the injection of user-generated images can not only prevent the collapse of the image model diversity but also ensure that the fidelity remains bounded over time. In other words, does the injection of external images drawn from a fixed reference distribution with \(\vmu_{\text{user}}(x_i) = \vmu_0(x_i)\) and \(\mSigma_{\text{user}}(x_i) = \mSigma_0(x_i)\) act as an effective regularizer that keeps the image model's output close to the reference distribution? \Cref{thm:boundedness_image_fidelity_injection} answers this question affirmatively by establishing an explicit bound on the long-term expected fidelity. In our analysis, we assume that exactly \(Np_i\) generated images are drawn from the current image model at each update, thereby removing the randomness associated with the number of images drawn for each text. This assumption also effectively corresponds to a potentially frozen text model. While the former assumption is not strictly essential, a more general analysis allowing a random number of images would yield analogous results at the cost of increased technical complexity.

\begin{restatable}[Boundedness of image model fidelity under image injection]{theorem}{BoundednessImageFidelityInjection}
\label{thm:boundedness_image_fidelity_injection}
Suppose that image injection procedure described in~\cref{alg:co-evolving_with_image_injection} is employed in the co-evolving training process. Assume that at each macro time step \(t\), exactly \(Np_i\) generated images are drawn from \(q_{t-1}(\cdot|x_i)\) (instead of \(N_i\) samples, where \(N_i\) is a random variable), and that the user-content images satisfy 
\begin{equation}
\vmu_{\text{user}}(x_i) = \vmu_0(x_i), \quad \mSigma_{\text{user}}(x_i) = \mSigma_0(x_i)
\end{equation}
for all \(1\leq i\leq K\). Suppose further that \(N_0>0\). Then, as \(t\to\infty\), the expected image model fidelity \(F_t(x_i)\) defined by~\labelcref{eq:image_model_fidelity} converges and is bounded by
\begin{equation}
\limsup_{t\rightarrow\infty}\E[F_t(x_i)] \leq \Big(\dfrac{1-(Np_i)^{-1}\lambda}{Np_i\lambda^{-1}-1-Np_i\lambda}\cdot \tr\big(\mSigma_0(x_i)\big)\Big)^{1/2},
\end{equation}
where \(\lambda = Np_i/(Np_i+N_0) < 1\).
\end{restatable}

\ifbool{sketch}{
\begin{proof}[Sketch of proof]
The proof proceeds in several key steps:
\begin{enumerate}
\item We first derive a recurrence for the expected squared fidelity, \(\E[F_{t+1}(x_i)^2]\), in terms of the previous fidelity \(\E[F_t(x_i)^2]\) and the traces \(\E\big[\tr(\mSigma_t(x_i))\big]\) and \(\tr\big(\mSigma_0(x_i)\big)\) arising from the sample averages.
\item Next, using the standard pooled covariance update formula, we derive a similar recurrence for the expected covariance trace, \(\E\big[\tr(\mSigma_{t+1}(x_i))\big]\), as a function of \(\E[F_t(x_i)^2]\), \(\E\big[\tr(\mSigma_t(x_i))\big]\), and \(\tr\big(\mSigma_0(x_i)\big)\).
\item By combining these two recurrences, we obtain a coupled linear system for \(\E[F_t(x_i)^2]\) and \(\E\big[\tr(\mSigma_t(x_i))\big]\). Standard analysis of this system shows that both sequences converge to finite limits. Solving for the limit and applying Cauchy--Schwarz inequality yields the stated upper bound on \(\limsup_{t\to\infty}\E[F_t(x_i)]\).
\end{enumerate}
\end{proof}
}

The explicit bound given in~\cref{thm:boundedness_image_fidelity_injection} demonstrates that the long-term expected fidelity is directly governed by the injection parameters. In particular, as the number \(N_0\) of injected images increases relative to the number \(Np_i\) of generated images, the bound on the fidelity becomes tighter. This implies that a larger injection batch not only helps to preserve the diversity of the image model but also keeps its output more closely aligned with the reference distribution. It is important to note that both this result and the boundedness result in~\cref{thm:image_model_mean_frozen_text_model} assume that the text model remains frozen. However, in~\cref{thm:image_model_mean_frozen_text_model} the boundedness of image model fidelity is achieved as a consequence of image model collapse; specifically, the exponential decay of the image model diversity plays a crucial role in ensuring that the fidelity converges to a small value. In contrast, the boundedness established in~\cref{thm:boundedness_image_fidelity_injection} does not rely on any such exponential decay property. Instead, it is derived solely through the injection of user-generated images drawn from a fixed reference distribution. This external injection acts as a continuous regularizer, ensuring that the image model's output remains close to the reference distribution regardless of the behavior of its diversity. See~\cref{subsubsec:exp_stabilization_of_the_image_model_via_user-content_injection} for a discussion of the experimental findings that corroborate~\cref{thm:stabilization_image_diversity_injection,thm:boundedness_image_fidelity_injection}.

%%%%%%%%%%%%%%%%%%%%%%%%%%%%%%%%%%%%%%%%%%%%%%%%%%%%
\section{Related Work}

In this section, we review research on self-consuming loops, organizing the discussion into empirical studies and theoretical studies. We then compare our co-evolving generative model training procedure (see~\cref{alg:co-evolving_generative_models}) with the classical Expectation-Maximization (EM) algorithm~\cite{mclachlan2008algorithm}.

%%%%%%%%%%%%%%%%%%%%%%%%%%%%%%%%%%%
\subsection{Empirical studies} Several empirical investigations have documented the adverse effects of training on synthetic data. The literature on this topic generally distinguishes between three data types: (\romannumeral 1) image data; (\romannumeral 2) text data; and (\romannumeral 3) multimodal data. Most prior studies have focused on the first two categories. For image data, prominent generative models such as variational autoencoders (VAEs)~\cite{shumailov2024ai}, generative adversarial networks (GANs)~\cite{alemohammad2024self,casco2023toward}, diffusion models (DMs)~\cite{bohacek2023nepotistically,martinez2023combining,martinez2023towards} have been investigated, all of which suggested that repeatedly feeding generative models with AI-generated data over time leads to degeneration such as unrealism~\cite{bohacek2023nepotistically,casco2023toward}, blurriness~\cite{martinez2023combining,martinez2023towards}, homogenization~\cite{martinez2023combining,martinez2023towards,shumailov2024ai}, or magnification of artifacts~\cite{alemohammad2024self,bohacek2023nepotistically}. For text data, self-consuming loop has known to cause decline in language diversity~\cite{guo2024curious,peterson2025ai} or amplifies biases~\cite{pan2024feedback,wyllie2024fairness}. In contrast to these studies, which largely focus on isolated modalities, \cite{conde2024analyzing} investigates a multimodal scenario where recursive modality changes such as repeatedly converting images to text and back result in a gradual degradation of the original content. Their findings show that generated images and descriptions drift away from the original input, eventually losing key semantic elements. While their work is based on data-driven observations and focuses on the inference loop, our work complements theirs by providing theoretical insights into the training loop. We explicitly model the mutual feedback between text and image models, thereby exploring how inter-modal reinforcement can accelerate collapse or, alternatively, be stabilized through external interventions.

%%%%%%%%%%%%%%%%%%%%%%%%%%%%%%%%%%%
\subsection{Theoretical studies} On the theoretical side, several studies have developed analytical frameworks to understand the underlying mechanisms driving collapse~\cite{dohmatob2024model,dohmatob2024tale,marchi2024heat,suresh2024rate} or mitigate it through incorporating fresh data~\cite{bertrand2024stability,gerstgrasser2024model} or leveraging reinforcement learning techniques~\cite{feng2024beyond,ferbach2024self,gillman2024self}. In addition, several case studies have examined collapse phenomena across a range of settings, including discrete distributions~\cite{suresh2024rate}, Gaussian distributions (both \(1\)-dimensional~\cite{suresh2024rate} and higher-dimensional cases~\cite{bertrand2024stability}), \(1\)-dimensional Gaussian mixture~\cite{suresh2024rate}, and regression problems~\cite{dohmatob2024model,dohmatob2024tale,taori2023data}. However, existing studies typically examine settings in which a model is trained solely on its own generated data. In contrast, given the growing prevalence of multimodal systems, our work investigates a co-evolving system in which two generative models are trained recursively on data produced by each other. In terms of theoretical contributions, our work extends and complements existing theory in the following ways: (\romannumeral 1) for the isolated case, we prove that the image model converges exponentially. This result not only confirms findings reported in~\cite{bertrand2024stability,suresh2024rate} but also extends them to higher-dimensional settings, beyond the \(1\)-dimensional case analyzed in~\cite{suresh2024rate}, and provides an explicit convergence rate, which improves upon the results in~\cite{bertrand2024stability}; (\romannumeral 2) regarding stabilization techniques, we demonstrate that incorporating fresh data into the training process prevents the co-evolving system from collapsing. This insight generalizes conclusions drawn for single-model settings~\cite{bertrand2024stability,gerstgrasser2024model} to the more complex scenario of interdependent, multimodal system.

%%%%%%%%%%%%%%%%%%%%%%%%%%%%%%%%%%%
\subsection{Comparison with the EM algorithm} The Expectation-Maximization (EM) algorithm\footnote{\label{fnote:em_background} Please refer to~\cref{sec:math_background} for an introduction to the EM algorithm as applied to the Gaussian Mixture Model.} is a classic iterative method used for maximum likelihood estimation in models with latent variables~\cite{mclachlan2008algorithm}. In each iteration, it alternates between computing the expected value of the latent variables given the observed data (the E-step) and updating the model parameters by maximizing the expected complete-data log-likelihood (the M-step). This structure of alternating between soft assignments of hidden variables and parameter updates is echoed in our co-evolving text-image system. Here, the text model update where the posterior probabilities of texts given the generated images are computed and averaged parallels the E-step, while the subsequent image model update where parameters such as the mean and covariance of the conditional Gaussian distributions are recalculated based on the new samples mirrors the M-step. Despite these similarities, significant differences distinguish our approach from EM. Unlike EM, which is designed to maximize a well-defined likelihood function over a fixed dataset, our co-evolving system continuously generates data through its own models, creating a dynamic feedback loop between the texts and images. This means that while EM relies on deterministic, closed-form computations to update responsibilities and parameters, our procedure uses Monte Carlo sampling to approximate these updates, introducing an element of stochasticity. Moreover, the bidirectional adaptation in our system, where both models influence each other iteratively, deviates from the unidirectional latent variable estimation framework of EM, resulting in a training procedure that prioritizes reinforcing engaging text-image pairs over strict likelihood maximization and its associated convergence guarantees.

%%%%%%%%%%%%%%%%%%%%%%%%%%%%%%%%%%%%%%%%%%%%%%%%%%%%
\section{Experiments}
\label{sec:experiments}

In this section, we present the experimental results from three perspectives: (\romannumeral 1) isolating the dynamics by freezing one model to understand its individual contribution (see~\cref{subsec:exp_isolating_the_dynamics_by_freezing_one_model}); (\romannumeral 2) image-driven acceleration in text model collapse (see~\cref{subsec:exp_image_driven_acceleration_in_text_model_collapse}); and (\romannumeral 3) stabilization strategies to counteract model collapse (see~\cref{subsec:exp_stabilization_strategies}). Unless otherwise noted, all experiments use \(K=5\) text components (for~\cref{subsubsec:exp_stabilization_of_the_text_model_via_corpus_injection}, the text model is initialized with \(K=5\) before injecting any text), an image dimension of \(d=2\), and \(N=1000\) samples for each training step. The image model is initialized with mean vectors distributed uniformly on the unit circle, and all reported results are averaged over \(100\) independent runs. Please refer to~\cref{fig:exp_overview} for an overview of the experimental setup\footnote{\label{ fnote:typical_run} We also include visualizations of text model histograms and image model generated samples from a typical run in~\cref{sec:visualizations_of_a_typical_run}.}.

\begin{figure}[htb]
\centering
\includegraphics[width=\linewidth]{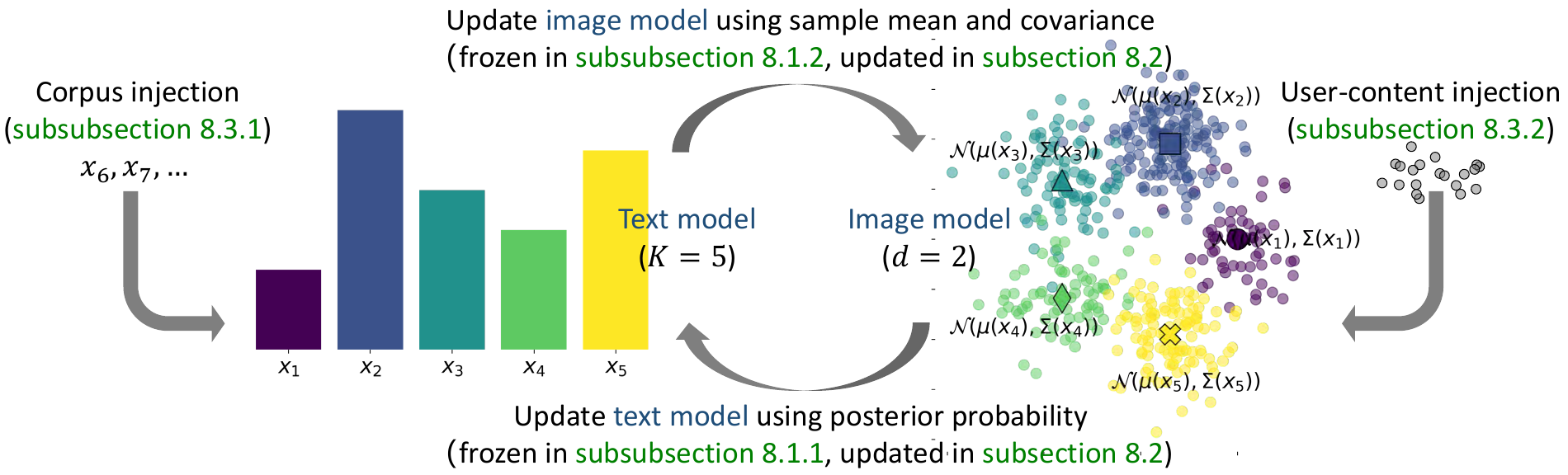}
\caption{An overview of the setup of experiments in~\cref{sec:experiments}. The text model is a discrete distribution over \(K=5\) text components, while the image model is a \(2\)-dimensional Gaussian mixture with \(K=5\) components whose means are uniformly distributed on the unit circle. The gray arrows indicate how generated texts are used to train the image model (via sample mean and covariance), how generated images are used to train the text model (via posterior probability), and how external information (e.g., corpus or user-content injection) can be introduced.}
\label{fig:exp_overview}
\end{figure}

%%%%%%%%%%%%%%%%%%%%%%%%%%%%%%%%%%%%%%%%%%%%%%%%%%%%
\subsection{Isolating the Dynamics by Freezing One Model}
\label{subsec:exp_isolating_the_dynamics_by_freezing_one_model}

In this section, we conduct two controlled experiments in which one of the models is held frozen, and we measure the evolution of the diversity measures defined in~\labelcref{eq:text_model_diversity,eq:image_model_diversity}.

%%%%%%%%%%%%%%%%%%
\subsubsection{Text Model Diversity under the Frozen Image Model}
\label{subsubsec:exp_text_model_diversity_under_the_frozen_image_model}

To isolate the effect of text updates, we hold the image model fixed (i.e., \(M_t=1\) and \(N_t=0\) in~\cref{alg:co-evolving_generative_models}). The image model is initialized with \(d=2\) by uniformly placing its mean vectors on the unit circle and setting all its covariance matrices to \(\sigma^2\mI\) with \(\sigma^2 \in \{0.01,0.1,0.5,1,10\}\). The text model is initialized uniformly over \(K=5\) texts and updated at each macro time step using \(N=1000\) samples, with results averaged over \(100\) independent runs. We plot text model diversity against the macro time step in~\cref{fig:exp_freeze_image_model_and_update_text_model}, using a logarithmic scale for the \(y\)-axis. As shown in~\cref{fig:exp_freeze_image_model_and_update_text_model}, text model diversity gradually decreases over time. Moreover, larger values of \(\sigma^2\) slow this decline, aligning with the predictions of~\cref{thm:text_model_diversity_frozen_image_model,thm:arbitrarily_small_convergence_large_covariance}.

\begin{figure}[htb]
\centering
\includegraphics[width=0.45\linewidth]{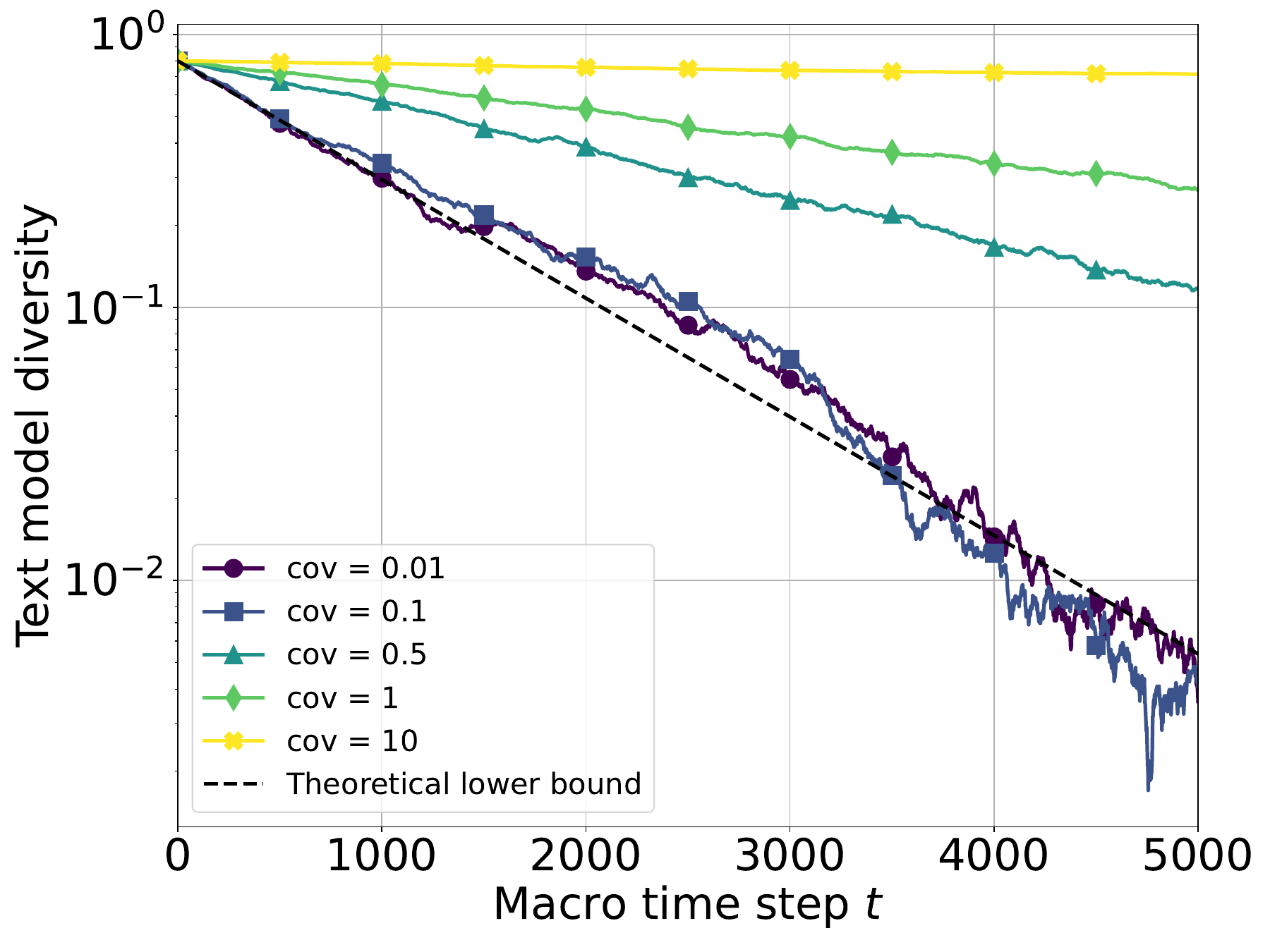}
\caption{Evolution of text model diversity when the image model is frozen. In this experiment, the image model is initialized with \(d=2\) by uniformly distributing its mean vectors on a unit circle, and its covariance matrices are set to \(\sigma^2\mI\) with \(\sigma^2\) taking on values of \(0.01\), \(0.1\), \(0.5\), \(1\), and \(10\). The text model is initialized with a uniform distribution over \(K=5\) texts and updated at each macro time step using a batch of \(N=1000\) samples; results are averaged over \(100\) independent runs. Different line colors represent different covariance scales \(\sigma^2\) (with lighter colors corresponding to larger values). The black dashed line indicate the theoretical lower bound of text model diversity convergence rate derived in~\cref{cor:lower_bound_text_model_diversity}, i.e., \(1-N^{-1}=0.999\). The results indicate that as \(\sigma^2\) increases, the rate at which text model diversity decreases becomes slower. Conversely, as \(\sigma^2\) decreases, the convergence rate increases, but is approximately bounded below by the theoretical lower bound. These are consistent with the theoretical predictions of~\cref{thm:text_model_diversity_frozen_image_model,thm:arbitrarily_small_convergence_large_covariance}.}
\label{fig:exp_freeze_image_model_and_update_text_model}
\end{figure}

%%%%%%%%%%%%%%%%%%
\subsubsection{Image Model Diversity under the Frozen Text Model}
\label{subsubsec:exp_image_model_diversity_under_the_frozen_text_model}

Next, we freeze the text model and update only the image model (i.e., \(M_t=0\) and \(N_t=1\) in~\cref{alg:co-evolving_generative_models}). In this experiment, the text model is initialized with \(\vp=(0.06,0.13,0.2,0.27,0.34)\) for \(K=5\) texts, and the image model is initialized as described in~\cref{subsubsec:exp_text_model_diversity_under_the_frozen_image_model} with \(s=1\). The image model is updated at each macro time step using a batch of \(N=1000\) samples, and the diversity values are averaged over \(100\) independent runs and plotted for \(T=2000\) macro times steps. \Cref{fig:exp_freeze_text_model_and_update_image_model} shows that the image model diversity decays exponentially over time, with a decay rate that approximates the theoretical prediction \(\rho(x_i) \approx 1-(d+1)/(8(N+1)p_i)\) given in~\cref{thm:diff_convergence_rate_image_model}.

\begin{figure}[htb]
\centering
\includegraphics[width=0.45\linewidth]{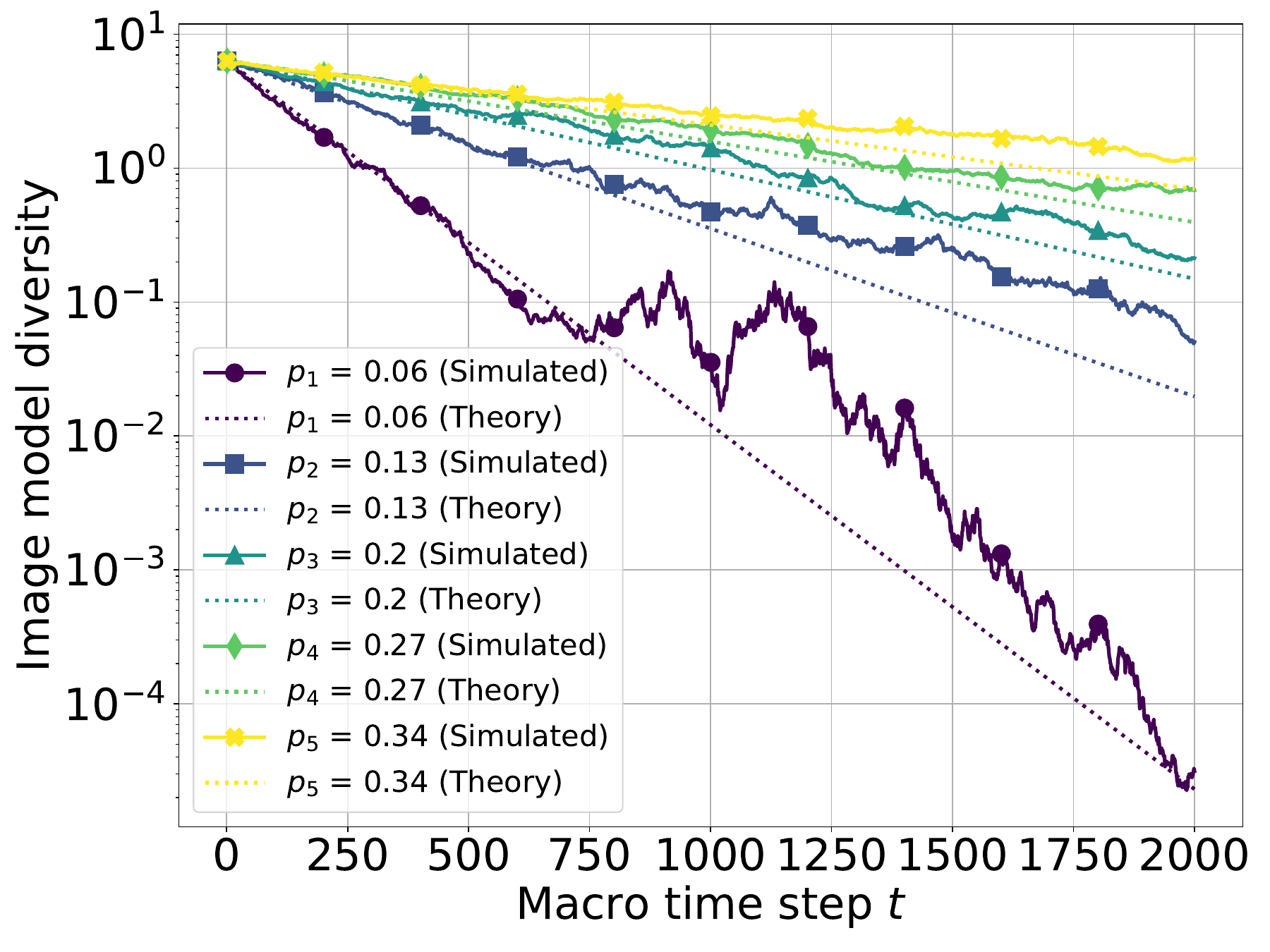}
\caption{Evolution of image model diversity evolution when the text model is frozen. Here, the text model is initialized with \(\vp=(0.06,0.13,0.2,0.27,0.34)\) (for \(K=5\) texts), and the image model is initialized as in~\cref{subsubsec:exp_text_model_diversity_under_the_frozen_image_model} with \(s=1\). The image model is updated using \(N=1000\) samples per macro time step, and results are averaged over \(100\) runs. Different line colors represent different text probabilities \(p_i\) (with lighter colors corresponding to larger values). The colored dashed lines indicate the convergence rate predicted by~\cref{thm:diff_convergence_rate_image_model} (i.e., \(\rho(x_i) \approx 1-(d+1)/(8(N+1)p_i)\)). The observed exponential decay in diversity is consistent with the theoretical prediction, especially when \(t\) is relatively small and less affected by numerical round-off errors; however, deviations tend to emerge at later times likely due to the accumulation of numerical inaccuracies.}
\label{fig:exp_freeze_text_model_and_update_image_model}
\end{figure}

%%%%%%%%%%%%%%%%%%%%%%%%%%%%%%%%%%%%%%%%%%%%%%%%%%%%
\subsection{Image-Driven Acceleration in Text Model Collapse}
\label{subsec:exp_image_driven_acceleration_in_text_model_collapse}

In this experiment, we explore how the image model influences the speed at which the text model collapses, as predicted by~\cref{thm:acceleratd_text_model_collapse}. To this end, we vary the number of image updates \(N_t = 0, 1, 2, 5, 10\) performed between successive text model updates in~\cref{alg:co-evolving_generative_models}, where \(N_t=0\) corresponds to a frozen image model. The results are averaged over \(100\) independent runs and plotted over \(T=1000\) macro time steps. Our findings, as illustrated in~\cref{fig:exp_acceleration_effect}, reveal that more frequent image updates result in a more rapid decrease in text model diversity, thereby validating the theoretical predictions.

\begin{figure}[htb]
\centering
\includegraphics[width=0.45\linewidth]{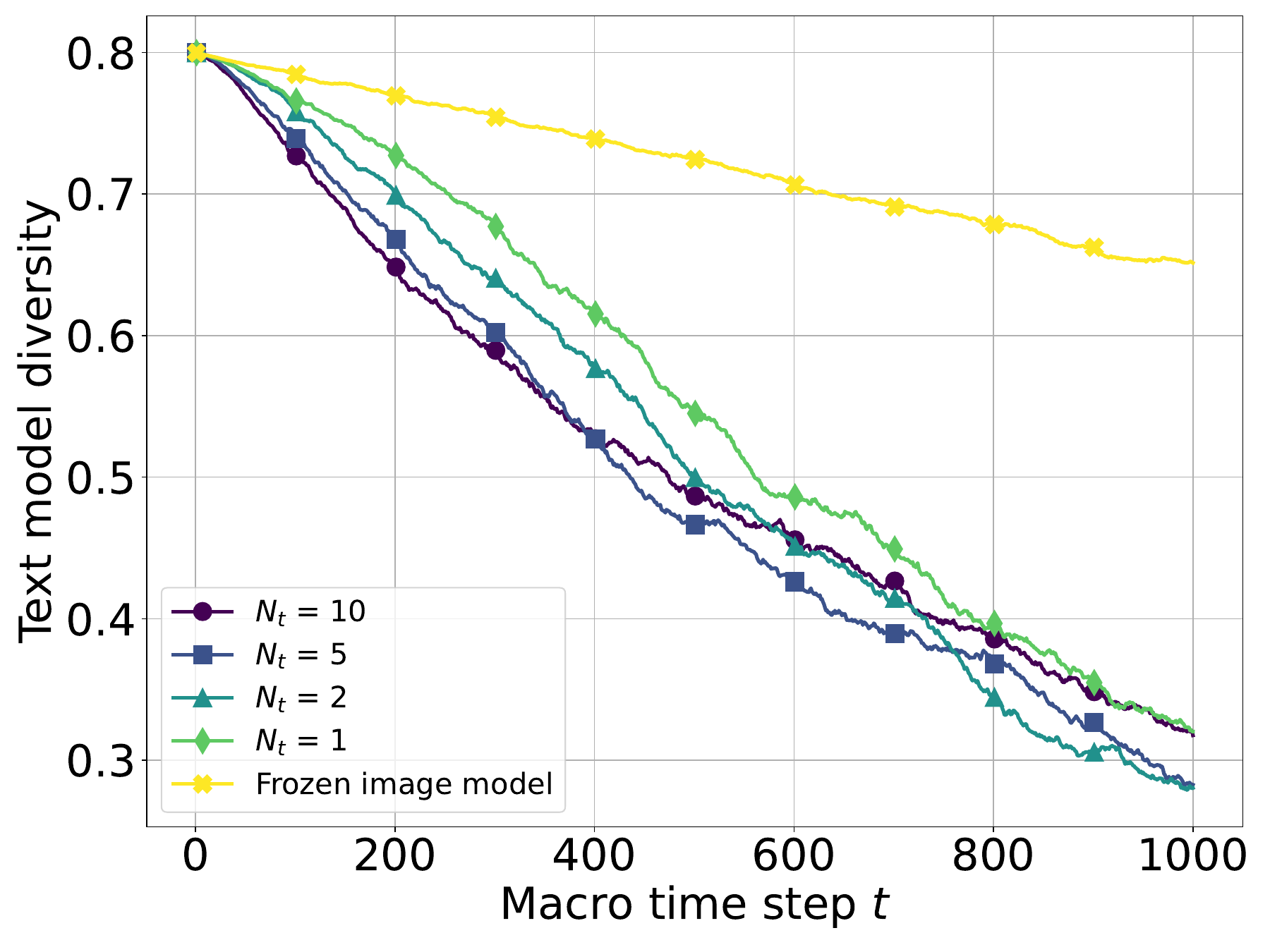}
\caption{Evolution of text model diversity under varying frequencies of image model updates. In this experiment, the image model is initialized as in~\cref{subsubsec:exp_text_model_diversity_under_the_frozen_image_model} with \(s=1\), and the text model is updated using \(N=1000\) samples per macro time step, with results averaged over \(100\) runs. Different line colors correspond to different numbers \(N_t\) of image updates performed between successive text model updates (with \(N_t=0\) indicating a frozen image model and lighter colors corresponding to smaller values). The results show that, compared with the frozen image model, more frequent image updates accelerate the collapse of the text model, in agreement with the theoretical prediction in~\cref{thm:acceleratd_text_model_collapse}.}
\label{fig:exp_acceleration_effect}
\end{figure}

%%%%%%%%%%%%%%%%%%%%%%%%%%%%%%%%%%%
\subsection{Stabilization Strategies}
\label{subsec:exp_stabilization_strategies}

In this section, we investigate two stabilization mechanisms designed to counteract the collapse observed in co-evolving generative models. Specifically, we examine (\romannumeral 1) stabilization of the text model via corpus injection and (\romannumeral 2) stabilization of the image model via user-content injection. Our experiments demonstrate that, by introducing external data into the training loops, both text diversity and image diversity can be preserved over time.

%%%%%%%%%%%%%%%%%%
\subsubsection{Stabilization of the Text Model via Corpus Injection}
\label{subsubsec:exp_stabilization_of_the_text_model_via_corpus_injection}

To counteract the collapse of the text model observed in closed systems, we introduced a stabilization mechanism via corpus injection in~\cref{subsec:stabilization_of_the_text_model_via_corpus_injection}. In this experiment, we keep the injection probability \(\alpha=0.05\), and vary the fraction of probability reallocated during corpus injection, i.e., \(\varepsilon=0.0, 0.01, 0.02, 0.05, 0.1\). At each macro time step, with probability \(\alpha\) a fraction \(\varepsilon\) of the existing text model probability is reallocated to a newly injected text. The simulation results over \(T=10000\) macro time steps in~\cref{fig:exp_corpus_injection}, averaged over \(100\) runs, confirm that corpus injection prevents the text model diversity from collapsing entirely, thereby ensuring a strictly positive diversity level over time. These empirical findings are consistent with the theoretical guarantees provided in~\cref{thm:stabilization_text_model_via_injection}.

\begin{figure}[htb]
\centering
\includegraphics[width=0.45\linewidth]{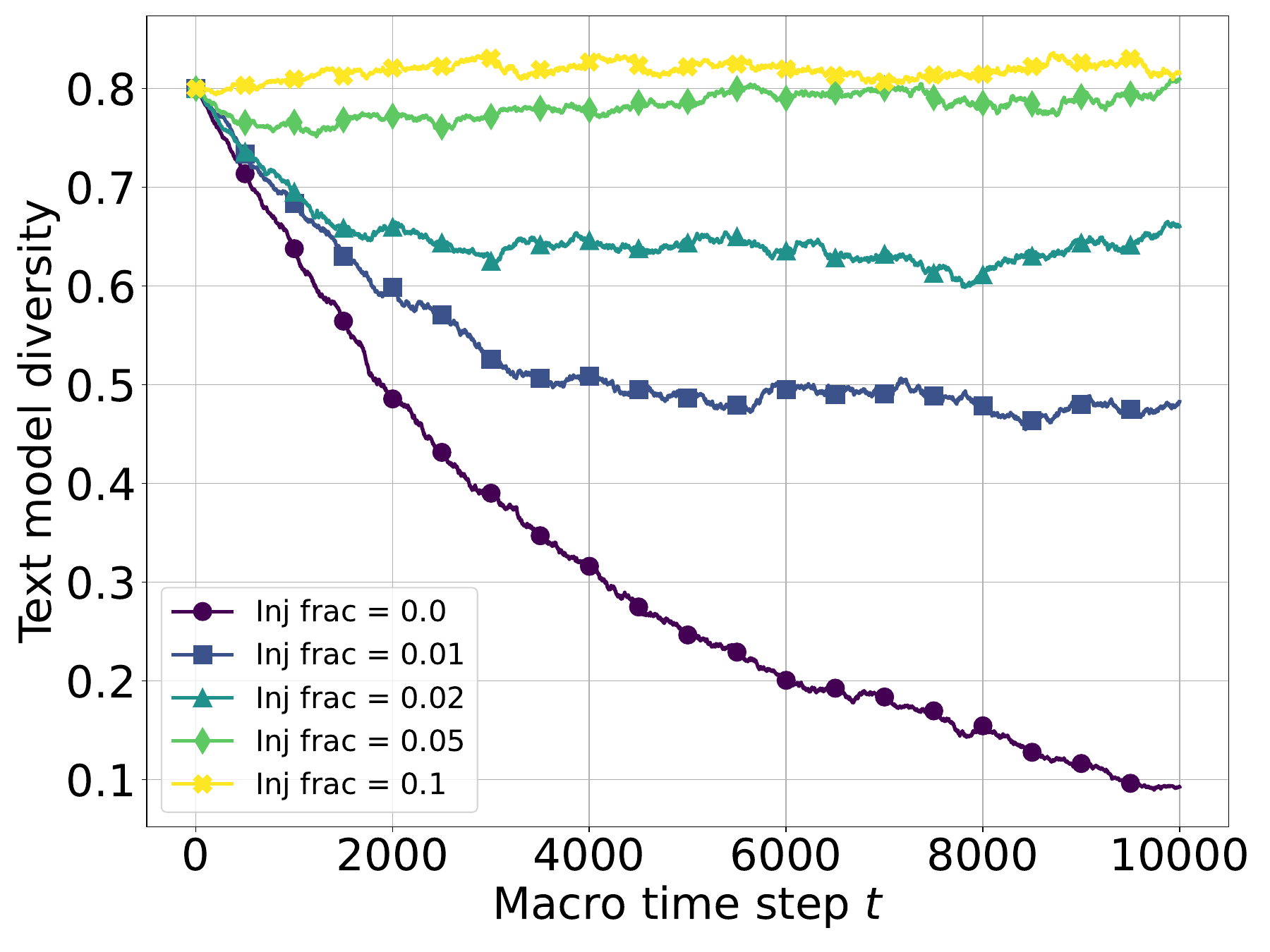}
\caption{Evolution of text model diversity under corpus injection. The plot displays the average diversity over macro time steps (with \(100\) runs) for different injection fractions \(\varepsilon\) (with fixed injection probability \(\alpha=0.05\)). Each curve corresponds to a different value of \(\varepsilon\), with colors becoming lighter as \(\varepsilon\) increases. The results illustrating that compared with the closed system (i.e., \(\varepsilon=0\)), even a small fraction of injected probability prevents the text model diversity from collapsing entirely, which are in line with the theoretical guarantees provided in~\cref{thm:stabilization_text_model_via_injection}.}
\label{fig:exp_corpus_injection}
\end{figure}

%%%%%%%%%%%%%%%%%%
\subsubsection{Stabilization of the Image Model via User-Content Injection}
\label{subsubsec:exp_stabilization_of_the_image_model_via_user-content_injection}

To counteract the collapse of the image model observed in closed systems, we introduced a stabilization mechanism via user-content injection in~\cref{subsec:stabilization_of_the_image_model_via_user-content_injection}. In this experiment, the image model is updated by drawing \(N=1000\) samples from the current model and, in addition, by injecting \(N_0\) images from an external distribution identical to the initial image model. The text model is fixed at \(K=1\). We vary \(N_0\) (e.g., \(N_0\in\{0, 1, 10, 100, 1000\}\)) and simulate the image model evolution over \(T=10000\) macro time steps, averaging results over \(100\) independent runs. Our results demonstrate that user-content injection prevents the collapse of image model diversity and maintains bounded fidelity, and fidelity decreases as \(N_0\) increases, in accordance with the theoretical guarantees provided in~\cref{thm:stabilization_image_diversity_injection,thm:boundedness_image_fidelity_injection}.

\begin{figure}[htb]
\centering
\includegraphics[width=0.45\linewidth]{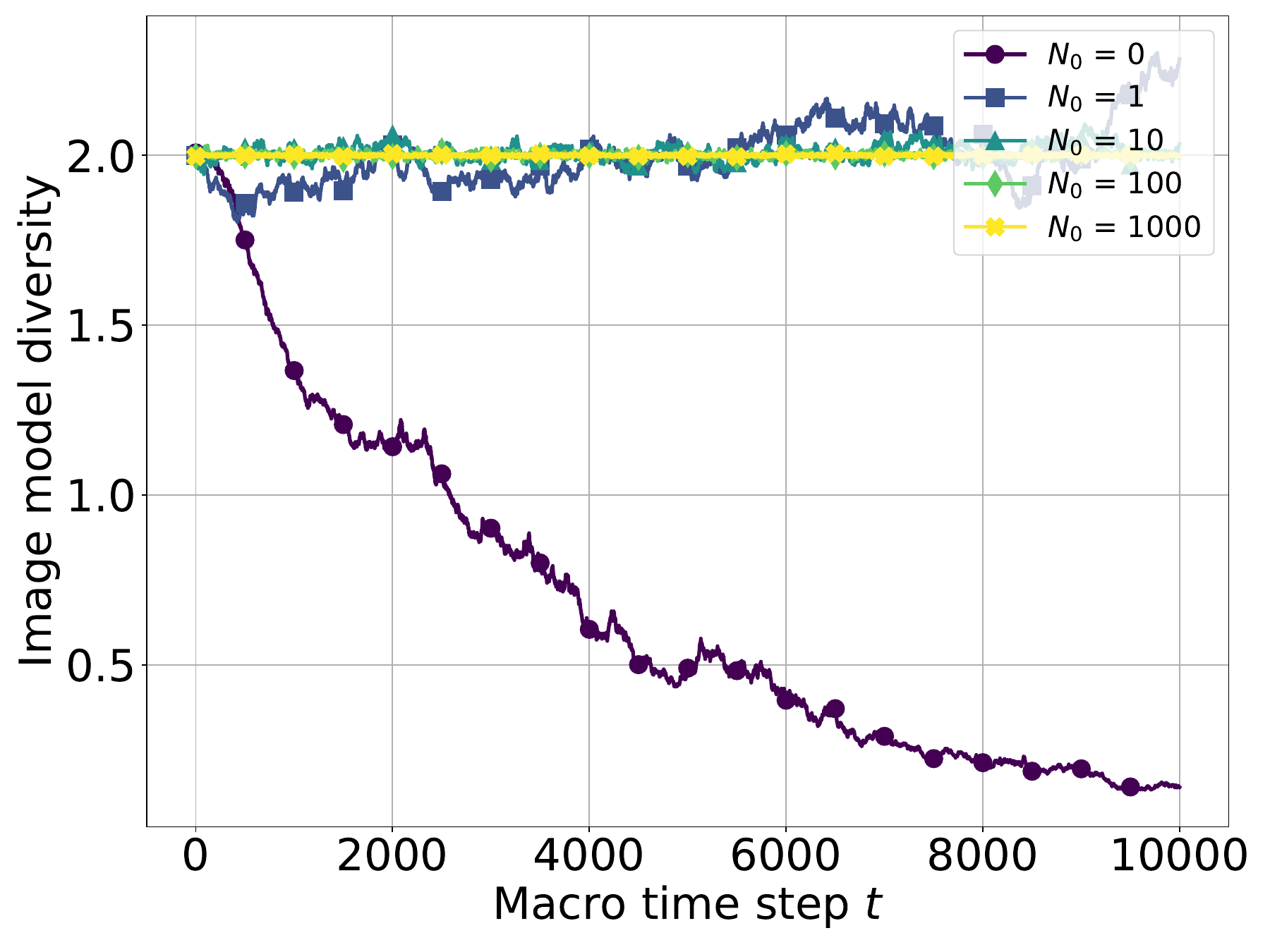}
\includegraphics[width=0.45\linewidth]{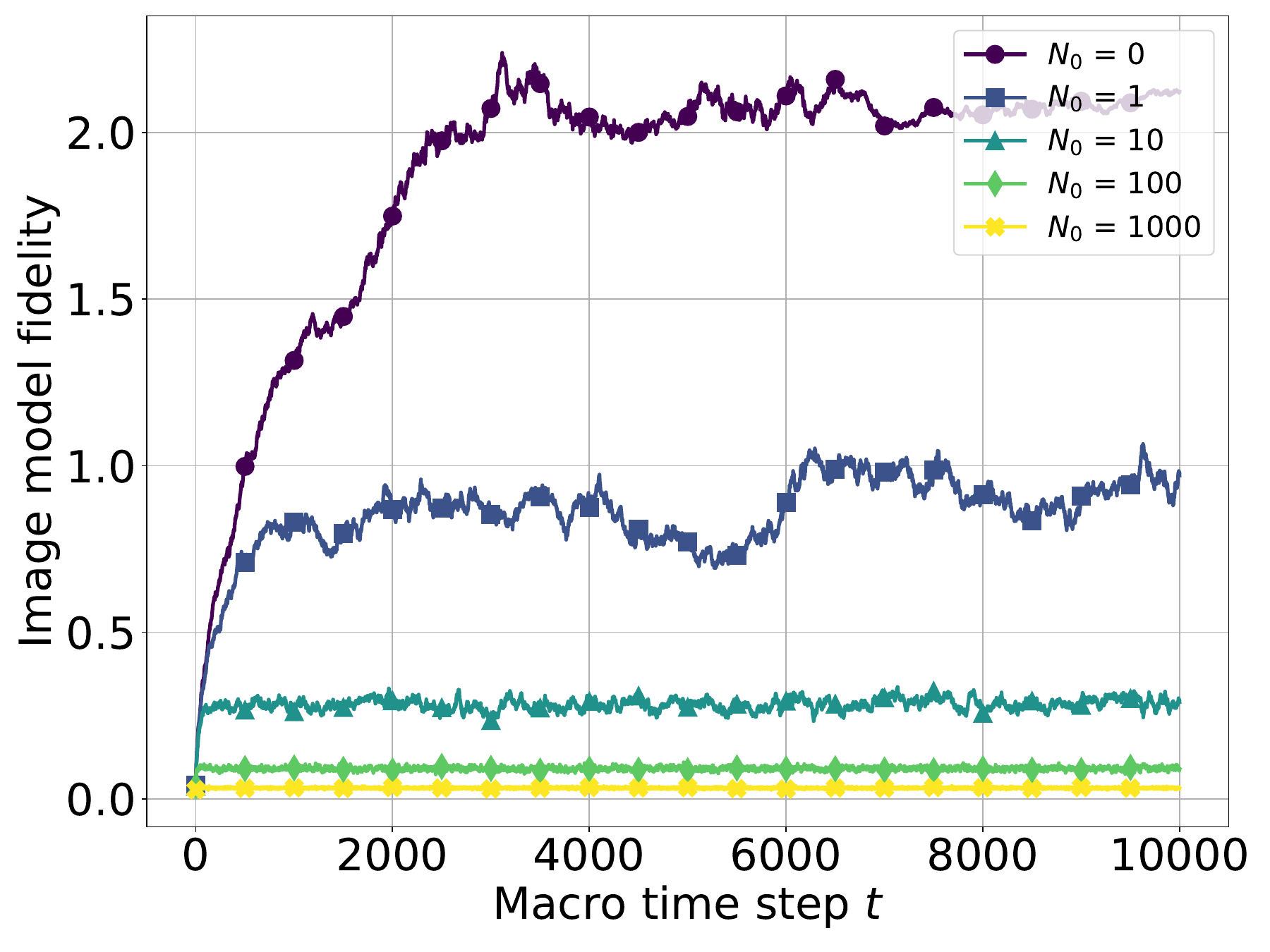}
\caption{User-content injection stabilizes the image model. Left: Evolution of image model diversity when the text model is fixed (\(K=1\)), with \(N=1000\) samples drawn per update and \(N_0\) images injected from an external distribution identical to the initial image model. Right: Corresponding evolution of image model fidelity. The curves, averaged over \(100\) independent runs, are shown for different values of \(N_0\) (\(N_0\in\{0, 1, 10, 100, 1000\}\)) with colors becoming lighter. These results demonstrate that injecting user content, even with an extremely small value (i.e., \(N_0 = 1\)) prevents the collapse of image model diversity and maintains bounded fidelity, in agreement with the theoretical guarantees provided in~\cref{thm:stabilization_image_diversity_injection,thm:boundedness_image_fidelity_injection}.}
\label{fig:exp_user_injection}
\end{figure}

%%%%%%%%%%%%%%%%%%%%%%%%%%%%%%%%%%%%%%%%%%%%%%%%%%%%
\section{Conclusion}

In this paper, we analyzed a co-evolving text–image system and provided insights into their mutual influence. Our analysis shows that (\romannumeral 1) when one model is frozen, the other collapses (with text diversity decaying monotonically and image covariance contracting exponentially); (\romannumeral 2) mutual feedback accelerates collapse via both a general acceleration effect and a Matthew effect that reinforces dominant texts; and (\romannumeral 3) external interventions such as corpus injections for text and user-content injections for images can stabilize the system, preserving both diversity and fidelity. These theoretical findings highlight the importance of continual exposure to fresh data to maintain long-term model robustness.

Our analysis naturally extends to additional modalities and more complex multi-model ecosystems. For instance, one can envision incorporating audio, video, or multimodal sensor data into a comprehensive generative ecosystem, where each model's outputs inform others in a network of training loops. While the general structure is preserved, differences in data representations and cross-modal feedback signals require refinements of our techniques. In particular, the latent spaces for modalities like audio or video differ from those for text and images, and the quantification of cross-modal dependencies become more intricate. Moreover, integrating multiple generative models simultaneously may require novel stabilization strategies to prevent any single modality from collapsing. Addressing these challenges will pave the way for more robust and adaptive multimodal AI systems.

%%%%%%%%%%%%%%%%%%%%%%%%%%%%%%%%%%%%%%%%%%%%%%%%%%%%
\bibliographystyle{plain}
\bibliography{auxiliaries/references}

%%%%%%%%%%%%%%%%%%%%%%%%%%%%%%%%%%%%%%%%%%%%%%%%%%%%
\clearpage
\paragraph{Roadmap.} The appendix is structured as follows:

\begin{itemize}
\item \Cref{sec:math_background} reviews several mathematical tools and concepts, which cover:
\begin{itemize}
\item The Wishart distribution (\cref{subsec:the_wishart_distribution})
\item Matrix square root (\cref{subsec:matrix_square_root})
\item Martingales and the martingale convergence theorem (\cref{subsec:martingale_convergence})
\item EM algorithm applied to Gaussian Mixture Models (\cref{subsec:em_algorithm})
\end{itemize}
\item \Cref{sec:proofs_to_theorems} provides detailed proofs of all the theorems in the main text, which include:
\begin{itemize}
\item Proofs of theorems in~\cref{sec:isolating_the_dynamics_by_freezing_one_model} (\cref{subsec:proofs_of_theorems_in_isolating})
\item Proofs of theorems in~\cref{sec:image-driven_acceleration_in_text_model_collapse} (\cref{subsec:proofs_of_theorems_in_acceleration})
\item Proofs of theorems in~\cref{sec:text-driven_matthew_effect_in_image_model_collapse} (\cref{subsec:proofs_of_theorems_in_matthew})
\item Proofs of theorems in~\cref{sec:stabilization_of_the_co-evolving_system} (\cref{subsec:proofs_of_theorems_in_stabilization})
\end{itemize}
\item \Cref{sec:visualizations_of_a_typical_run} presents visualizations of text and image model output, which include:
\begin{itemize}
\item Freeze the image model and update the text model (\cref{subsec:freeze_the_image_model_and_update_the_text_model})
\item Freeze the text model and update the image model (\cref{subsec:freeze_the_text_model_and_update_the_image_model})
\item Update both the text model and the image model (\cref{subsec:update_both_the_text_model_and_the_image_model})
\item Stabilize the text model via corpus injection (\cref{subsec:stabilize_the_text_model_via_corpus_injection})
\item Stabilize the image model via user-content injection (\cref{subsec:stabilize_the_image_model_via_user-content_injection})
\end{itemize}
\end{itemize}

%%%%%%%%%%%%%%%%%%%%%%%%%%%%%%%%%%%%%%%%%%%%%%%%%%%%
\appendix
\section{Mathematical Background}
\label{sec:math_background}

In this section, we review several mathematical tools and concepts used throughout our analysis. Specifically, we provide self-contained discussions on the Wishart distribution, the computation of the matrix square root and its operator concavity property, the concept of martingales and the martingale convergence theorem, and the Expectation-Maximization (EM) algorithm as applied to Gaussian Mixture Models.

%%%%%%%%%%%%%%%%%%%%%%%%%%%%%%%%%%%
\subsection{The Wishart Distribution}
\label{subsec:the_wishart_distribution}

The Wishart distribution~\cite{livan2018introduction} arises naturally in the context of sample covariance matrices. Let \(\vx_1, \vx_2, \dotsc, \vx_n \in \mathbb{R}^d\) be independent random vectors with
\begin{equation}
\vx_i \sim \mathcal{N}(\vzero, \mSigma),
\end{equation}
where \(\mSigma\in\mathbb{R}^{d\times d}\) is a positive definite matrix. Define the matrix
\begin{equation}
\mW = \sum_{i=1}^n \vx_i \vx_i^\top.
\end{equation}
Then, \(\mW\) follows a Wishart distribution with \(n\) degrees of freedom and \(\mSigma\) scale matrix, denoted as
\begin{equation}
\mW \sim \text{Wishart}_d(\mSigma, n).
\end{equation}
A fundamental property of the Wishart distribution is its expectation:
\begin{equation}
\E[\mW] = n \mSigma.
\end{equation}

\begin{proof}
By the linearity of expectation,
\begin{equation}
\E[\mW] = \sum_{i=1}^n \E[\vx_i \vx_i^\top].
\end{equation}
For each \(\vx_i \sim \mathcal{N}(\vzero, \mSigma)\), we have \(\E[\vx_i \vx_i^\top] = \mSigma\). Therefore,
\begin{equation}
\E[\mW] = n \mSigma,
\end{equation}
which completes the proof.
\end{proof}

In our analysis, we frequently consider the case where \(\mSigma = \mI\) (the identity matrix) and set \(n = N_i - 1\) (with \(N_i \geq 2\)), so that
\begin{equation}
\mW \sim \text{Wishart}_d(\mI, N_i - 1), \quad \text{and }\quad \E[\mW] = (N_i-1)\mI.
\end{equation}
This property is crucial when relating the updated sample covariance
\begin{equation}
\mSigma_{t+1}(x_i) = \mSigma_t(x_i)^{1/2} \cdot \frac{\mW}{N_i - 1} \cdot \mSigma_t(x_i)^{1/2}
\end{equation}
to the previous covariance \(\mSigma_t(x_i)\).

%%%%%%%%%%%%%%%%%%%%%%%%%%%%%%%%%%%
\subsection{Matrix Square Root}
\label{subsec:matrix_square_root}

For any positive definite matrix \(\mA\), the matrix square root \(\mA^{1/2}\) is defined as the unique positive definite matrix satisfying
\begin{equation}
\mA^{1/2} \cdot \mA^{1/2} = \mA.
\end{equation}
If \(\mA\) has the eigen-decomposition
\begin{equation}
\mA = \mQ \mLambda \mQ^\top,
\end{equation}
where \(\mQ\) is an orthogonal matrix and \(\mLambda\) is a diagonal matrix with positive entries, then the matrix square root is given by
\begin{equation}
\mA^{1/2} = \mQ \mLambda^{1/2} \mQ^\top,
\end{equation}
where \(\mLambda^{1/2}\) is the diagonal matrix whose entries are the square roots of those in \(\mLambda\).

A key property of the matrix square root is its operator concavity: for any two positive definite matrices \(\mA\) and \(\mB\) and any \(\lambda \in [0,1]\),
\begin{equation}
\big(\lambda \mA + (1 - \lambda) \mB\big)^{1/2} \succeq \lambda \mA^{1/2} + (1 - \lambda) \mB^{1/2}.
\end{equation}

\begin{proof}
Start with the basic inequality
\begin{equation}
(\mA^{1/2} - \mB^{1/2})^2 \succeq \vzero.
\end{equation}
Expanding this expression yields
\begin{equation}
\mA + \mB - \mA^{1/2}\mB^{1/2} - \mB^{1/2}\mA^{1/2} \succeq \vzero,
\end{equation}
which simplifies to
\begin{equation}
\mA^{1/2}\mB^{1/2} + \mB^{1/2}\mA^{1/2} \preceq \mA + \mB.
\end{equation}
Since \(\lambda - \lambda^2 = \lambda(1 - \lambda)\) is non-negative for \(\lambda \in [0,1]\), we have
\begin{equation}
(\lambda - \lambda^2)(\mA^{1/2}\mB^{1/2} + \mB^{1/2}\mA^{1/2}) \preceq (\lambda - \lambda^2)(\mA + \mB).
\end{equation}
Adding \(\lambda^2 \mA + (1 - \lambda)^2 \mB\) to both sides, we obtain
\begin{equation}
\lambda^2 \mA + (1 - \lambda)^2 \mB + (\lambda - \lambda^2)(\mA^{1/2}\mB^{1/2} + \mB^{1/2}\mA^{1/2}) \preceq \lambda \mA + (1 - \lambda) \mB.
\end{equation}
The left-hand side can be factored as
\begin{equation}
\big(\lambda \mA^{1/2} + (1 - \lambda) \mB^{1/2}\big)^2 \preceq \lambda \mA + (1 - \lambda) \mB.
\end{equation}
Since the square root function is operator monotone, applying it to both sides yields
\begin{equation}
\lambda \mA^{1/2} + (1 - \lambda) \mB^{1/2} \preceq \big(\lambda \mA + (1 - \lambda) \mB\big)^{1/2},
\end{equation}
which completes the proof.
\end{proof}

This property, along with Jensen's operator inequality, is crucial in our proofs when deriving bounds on expectations involving \(\mA^{1/2}\), especially when \(\mA\) is a random matrix (e.g., when \(\mA\) follows a Wishart distribution).

%%%%%%%%%%%%%%%%%%%%%%%%%%%%%%%%%%%
\subsection{Martingales and the Martingale Convergence Theorem}
\label{subsec:martingale_convergence}

A sequence of random variables \(\{X_t\}\) is called a martingale with respect to a filtration \(\{\mathcal{F}_t\}\) if it satisfies~\cite{hall2014martingale}:
\begin{enumerate}
\item \(X_t\) is \(\mathcal{F}_t\)-measurable,
\item \(\E[|X_t|] < \infty\), and
\item \(\E[X_{t+1} | \mathcal{F}_t] = X_t\) for all \(t\).
\end{enumerate}

The martingale convergence theorem~\cite{hall2014martingale} is stated as follows. Let \(\{X_t\}\) be a martingale that is uniformly integrable; that is,
\begin{equation}
\sup_t \E[|X_t|] < \infty.
\end{equation}
Then, there exists a random variable \(X_\infty\) such that \(X_t \to X_\infty\) almost surely and in \(L^1\) as \(t \to \infty\).

In our work, the sequence of probability vectors \(\{\vp_t\}\), where \(\vp_t = (p_t(x_1), \dots, p_t(x_K))\), forms a vector-valued martingale (with the filtration being \(\{\sigma(\vp_t)\}\), i.e., the \(\sigma\)-algebra generated by \(\vp_t\)) because the update rules for the text model preserve the expected value. Since each \(\vp_t\) lies in the \(K\)-simplex, which is compact, the martingale convergence theorem guarantees that \(\vp_t\) converges almost surely to a limit \(\vp_\infty\). This result is fundamental in establishing the long-term behavior of the text model.

%%%%%%%%%%%%%%%%%%%%%%%%%%%%%%%%%%%
\subsection{Expectation-Maximization Algorithm Applied to Gaussian Mixture Models}
\label{subsec:em_algorithm}

The Expectation-Maximization (EM) algorithm is an iterative method for finding maximum likelihood estimates in the presence of latent variables~\cite{mclachlan2008algorithm}. Consider a Gaussian Mixture Model (GMM) with \(K\) components, where the probability density function of an observation \(\vx \in \mathbb{R}^d\) is given by
\begin{equation}
p(\vx) = \sum_{k=1}^K \pi_k \, \mathcal{N}(\vx; \vmu_k, \mSigma_k),
\end{equation}
with mixing coefficients \(\pi_k\), means \(\vmu_k\), and covariance matrices \(\mSigma_k\). The EM algorithm proceeds by alternating between:

\begin{itemize}
\item E-step: Compute the posterior probabilities (responsibilities)
\begin{equation}
\gamma_{ik} = \frac{\pi_k \, \mathcal{N}(\vx_i; \vmu_k, \mSigma_k)}{\sum_{j=1}^K \pi_j \, \mathcal{N}(\vx_i; \vmu_j, \mSigma_j)},
\end{equation}
which represent the probability that observation \(\vx_i\) originates from component \(k\).
\item M-step: Update the parameters using these responsibilities:
\begin{equation}
\pi_k^{\text{new}} = \frac{1}{n} \sum_{i=1}^n \gamma_{ik},
\end{equation}
\begin{equation}
\vmu_k^{\text{new}} = \frac{\sum_{i=1}^n \gamma_{ik} \vx_i}{\sum_{i=1}^n \gamma_{ik}},
\end{equation}
\begin{equation}
\mSigma_k^{\text{new}} = \frac{\sum_{i=1}^n \gamma_{ik} (\vx_i - \vmu_k^{\text{new}})(\vx_i - \vmu_k^{\text{new}})^\top}{\sum_{i=1}^n \gamma_{ik}}.
\end{equation}
\end{itemize}

Our co-evolving system employs a similar alternating update scheme. The text model update computes soft assignments (analogous to the E-step) by estimating posterior probabilities based on generated images. Subsequently, the image model update re-estimates the parameters (reminiscent of the M-step) using these assignments. Although our updates are performed in a continuously evolving, stochastic setting, the conceptual similarity to the EM algorithm in Gaussian mixtures provides valuable intuition for understanding the iterative dynamics of our models.

%%%%%%%%%%%%%%%%%%%%%%%%%%%%%%%%%%%%%%%%%%%%%%%%%%%%
\section{Proofs of Theorems}
\label{sec:proofs_to_theorems}

In this section, we provide detailed proofs of the theorems.

%%%%%%%%%%%%%%%%%%%%%%%%%%%%%%%%%%%
\subsection{Proofs of Theorems in Section~3}
\label{subsec:proofs_of_theorems_in_isolating}

\TextModelDiversityFrozenImageModel*
\begin{proof}
The text model is updated according to the empirical posterior estimate
\begin{equation}
p_{t+1}(x_i) = \frac{1}{N} \sum_{j=1}^{N} p_t(x_i | \vy^{(j)}),
\end{equation}
where each \(\vy^{(j)}\) is sampled from the mixture distribution
\begin{equation}
r_t(\vy) = \sum_{k=1}^{K} p_t(x_k) q(\vy | x_k).
\end{equation}
Using the definition of the posterior,
\begin{equation}
p_t(x_i | \vy^{(j)}) = Z_i(\vy^{(j)}) \coloneqq \frac{p_t(x_i) q(\vy^{(j)} | x_i)}{\sum_{k=1}^{K} p_t(x_k) q(\vy^{(j)} | x_k)}.
\end{equation}
Taking expectations over the randomness in \(\vy\), we obtain
\begin{equation}
\E_{\vy \sim r_t} [Z_i(\vy)|\vp_t] = p_t(x_i),
\end{equation}
and for the second moment,
\begin{equation}
\E_{\vy\sim r_t} [Z_i(\vy)^2|\vp_t] = p_t(x_i)^2 + \Var_{\vy\sim r_t}(Z_i(\vy)|\vp_t),
\end{equation}
with \(\Var_{\vy\sim r_t}(Z_i(\vy)|\vp_t) \geq 0\). Since the updates are computed from \(N\) independent samples, the expectation of the squared probabilities follows from the independence property:
\begin{equation}
\E[p_{t+1}(x_i)^2|\vp_t] = \frac{1}{N}\cdot\E_{\vy\sim r_t}[Z_i(\vy)^2|\vp_t] + \frac{N-1}{N}\cdot p_t(x_i)^2.
\end{equation}
Summing over all \(i\), we obtain
\begin{equation}
\E\Big[\sum_{i=1}^{K} p_{t+1}(x_i)^2\Big|\vp_t\Big] = \sum_{i=1}^{K} p_t(x_i)^2 + \frac{1}{N} \sum_{i=1}^{K} \Big(\E_{\vy\sim r_t}[Z_i(\vy)^2|\vp_t] - p_t(x_i)^2\Big).
\end{equation}
By substituting this into the definition of \(H_{t+1}(\vp_{t+1}) = 1 - \sum_{i=1}^{K} p_{t+1}(x_i)^2\), we obtain
\begin{equation}
\E[H_{t+1}(\vp_{t+1})|\vp_t] = \Big(1-\frac{1}{N}\Big)\cdot H_t(\vp_t) + \frac{1}{N}\cdot\Big(1-\sum_{i=1}^{K} \E_{\vy\sim r_t}[Z_i(\vy)^2|\vp_t]\Big).
\end{equation}
Since for each \(i\) we have \(\E[Z_i(\vy)^2|\vp_t] \geq p_t(x_i)^2\), it follows that
\begin{equation}
\E[H_{t+1}(\vp_{t+1})|\vp_t] \leq H_t(\vp_t).
\end{equation}
This proves that the diversity measure is non-increasing in expectation. Moreover, as \(t\) tends to infinity, the system converges to an equilibrium distribution \(p_\infty\) for which
\begin{equation}
\sum_{i=1}^{K}\E_{\vy\sim r_t}[Z_i(\vy)^2|\vp_{\infty}] = \sum_{i=1}^{K} p_\infty(x_i)^2.
\end{equation}
To analyze the equality condition, note that for each \(i\) we have
\begin{equation}
\E_{\vy\sim r_t}[Z_i(\vy)^2|\vp_{\infty}] = \int \frac{p_\infty(x_i)^2 q(\vy|x_i)^2}{\sum_{k=1}^{K} p_\infty(x_k) q(\vy|x_k)} \dif\vy.
\end{equation}
Applying the Cauchy--Schwarz inequality, we obtain
\begin{equation}
\int \frac{p_\infty(x_i)^2 q(\vy|x_i)^2}{\sum_{k=1}^{K} p_\infty(x_k) q(\vy|x_k)} \dif\vy \cdot \int \sum_{k=1}^{K} p_\infty(x_k) q(\vy|x_k) \dif\vy\geq \Big(\int p_\infty(x_i) q(\vy|x_i) \dif\vy\Big)^2.
\end{equation}
Since \(\int \sum_{k=1}^{K} p_\infty(x_k) q(\vy|x_k) \dif\vy = 1\) and \(\int p_\infty(x_i) q(\vy|x_i) \dif\vy = p_\infty(x_i)\), it follows that
\begin{equation}
\int \frac{p_\infty(x_i)^2 q(\vy|x_i)^2}{\sum_{k=1}^{K} p_\infty(x_k) q(\vy|x_k)} \dif\vy \geq p_\infty(x_i)^2.
\end{equation}
Summing over \(i\) yields
\begin{equation}
\sum_{i=1}^{K}\E_{\vy\sim r_t}[Z_i(\vy)^2|\vy_{\infty}] \geq \sum_{i=1}^{K} p_\infty(x_i)^2.
\end{equation}
Equality holds if and only if equality holds in the Cauchy--Schwarz inequality for each \(i\), which requires that there exists a constant \(c_i\) (depending on \(i\)) such that
\begin{equation}
p_\infty(x_i) q(\vy|x_i) = c_i \Big(\sum_{k=1}^{K} p_\infty(x_k) q(\vy|x_k)\Big)
\end{equation}
for almost every \(\vy\). Integrating both sides over \(\vy\) gives
\begin{equation}
p_\infty(x_i) = c_i,
\end{equation}
so that the equality condition becomes
\begin{equation}
p_\infty(x_i) q(\vy|x_i) = p_\infty(x_i) \Big(\sum_{k=1}^{K} p_\infty(x_k) q(\vy|x_k)\Big).
\end{equation}
For any \(i\) with \(p_\infty(x_i)>0\) this implies
\begin{equation}
q(\vy|x_i) = \sum_{k=1}^{K} p_\infty(x_k) q(\vy|x_k).
\end{equation}
In other words, equality is attained if and only if the conditional distributions \(q(\vy|x_i)\) are identical for all \(i\) with \(p_\infty(x_i)>0\). This yields the two equilibrium cases: either the text model collapses to a one-hot vector (in which \(H_{\infty}(\vp_\infty)=0\)) or the image model assigns identical distributions \(q(\vy|x_i)\) to all texts. 
\end{proof}

\ImageModelCovarianceFrozenTextModel*
\begin{proof}
For a fixed text \(x_i\) and iteration \(t\) at which \(N_i\geq 2\), note that if
\begin{equation}
\vy^{(1)},\dotsc,\vy^{(N_i)} \sim \mathcal{N}\big(\vmu_t(x_i),\mSigma_t(x_i)\big),
\end{equation}
then it is well known (see, e.g., properties of the Wishart distribution~\cite{livan2018introduction}) that the sample covariance
\begin{equation}
\mSigma_{t+1}(x_i) = \frac{1}{N_i-1}\sum_{j=1}^{N_i}\big(\vy^{(j)}-\overline{\vy}\big)\big(\vy^{(j)}-\overline{\vy}\big)^\top
\end{equation}
where \(\overline{\vy} = (\vy^{(1)}+\dotsb+\vy^{(N_i)})/N_i\) satisfies
\begin{equation}
\mSigma_{t+1}(x_i) = \mSigma_t(x_i)^{1/2}\cdot\frac{\mW}{N_i-1}\cdot\mSigma_t(x_i)^{1/2}.
\end{equation}
Here, \(\mW \sim \text{Wishart}_d(\mI, N_i-1)\) is a \(d\)-dimensional Wishart random matrix with \((N_i-1)\) degrees of freedom and identity scale matrix\footnote{\label{fnote:wishart_background} Refer to~\cref{sec:math_background} for a brief introduction to the Wishart distribution.}. It is a standard fact that
\begin{equation}
\E\Big[\dfrac{\mW}{N_i-1}\Big] = \mI.
\end{equation}
Note that the function \(\mX\mapsto \mX^{1/2}\) is operator concave on the set of positive definite matrices\footnote{\label{fnote:matrix_square_root_background} Refer to~\cref{sec:math_background} for definition of matrix square root and the proof to this operator concavity.} By Jensen's operator inequality we have that the strict inequality holds, i.e.,
\begin{equation}
\E\Big[\Big(\dfrac{\mW}{N_i-1}\Big)^{1/2}\Big] \prec \mI.
\end{equation}
Therefore, the largest eigenvalue of \(\E[(\mW/(N_i-1))^{1/2}]\), say \(\rho(N_i)\) (depending on \(N_i\)), is smaller than \(1\). Consequently, we have
\begin{equation}
\begin{aligned}
\E\big[\tr\big(\mSigma_{t+1}^{1/2}(x_i)\big)\big|\mSigma_{t}(x_i), N_i\big] &= \E\Big[\tr\Big(\mSigma_t(x_i)^{1/2}\cdot\frac{\mW}{N_i-1}\cdot\mSigma_t(x_i)^{1/2}\Big)^{1/2}\Big|\mSigma_t(x_i), N_i\Big]\\
&=\E\Big[\sum_{j=1}^{d}\sigma_j\Big(\Big(\dfrac{\mW}{N_i-1}\Big)^{1/2}\cdot\mSigma_t(x_i)^{1/2}\Big)\Big|\mSigma_t(x_i), N_i\Big]\\
&\leq \E\Big[\sum_{j=1}^{N}\sigma_1\Big(\Big(\dfrac{\mW}{N_i-1}\Big)^{1/2}\Big)\cdot\sigma_j\big(\mSigma_{t}(x_i)^{1/2}\big)\Big|\mSigma_t(x_i), N_i\Big]\\
&\leq \rho(N_i)\cdot\E\big[\tr\big(\mSigma_{t}(x_i)^{1/2}\big)\big|\mSigma_t(x_i), N_i\big],
\end{aligned}
\end{equation}
where \(\sigma_j(\cdot)\) represents the \(j\)th largest singular value of a matrix. Since the probability of \(N_i\geq 2\) is positive, the expectation of \(\rho(N_i)\) is strictly less than \(1\). Applying the total expectation formula and iterating the inequality, we obtain
\begin{equation}
\E[D_t(x_i)] = \E\big[\tr\big(\mSigma_t(x_i)^{1/2}\big)\big]\leq C\rho^t,
\end{equation}
where \(C = \E\big[\tr\big(\mSigma_0(x_i)^{1/2}\big)\big]\) and \(\rho = \E_{N_i\sim\text{Binomial}(N, p_i)}[\rho(N_i)]<1\).
\end{proof}

\ImageModelMeanFrozenTextModel*
\begin{proof}
For a fixed text \(x_i\), and iteration \(t\) at which \(N_i\geq 1\), note that if
\begin{equation}
\vy^{(1)},\dotsc,\vy^{(N_i)} \sim \mathcal{N}\big(\vmu_t(x_i),\mSigma_t(x_i)\big),
\end{equation}
then the sample mean satisfies
\begin{equation}
\vmu_{t+1}(x_i)\sim \mathcal{N}\Big(\vmu_t(x_i), \dfrac{1}{N_i}\cdot\mSigma_t(x_i)\Big). 
\end{equation}
Thus, conditioning on \(\vmu_t(x_i)\) and \(N_i\), the difference between consecutive mean vectors follows
\begin{equation}
\vmu_{t+1}(x_i) - \vmu_t(x_i)\sim\mathcal{N}\Big(\vzero, \dfrac{1}{N_i}\cdot\mSigma_t(x_i)\Big).
\end{equation}
From the properties of the Gaussian distribution, we obtain
\begin{equation}
\E[\|\vmu_{t+1}(x_i) - \vmu_t(x_i)\|_2^2|\vmu_t(x_i), \mSigma_t(x_i), N_i] = \dfrac{1}{N_i}\cdot\E\big[\tr\big(\mSigma_t(x_i)\big)\big|N_i\big],
\end{equation}
Take expectation over \(N_i\sim\text{Binomial}(N, p_i)\), we derive
\begin{equation}
\begin{aligned}
\E[\|\vmu_{t+1}(x_i) - \vmu_t(x_i)\|_2^2|\vmu_t(x_i), \mSigma_t(x_i)]
&\leq \sum_{N_i=1}^{N}\dfrac{1}{N_i}\binom{N}{N_i}p_i^{N_i}(1-p_i)^{N-N_i}\cdot C^2\rho^{2t}\\
&\leq \sum_{N_i=1}^{N}\dfrac{2}{N_i+1}\binom{N}{N_i}p_i^{N_i}(1-p_i)^{N-N_i}\cdot C^2\rho^{2t}\\
&\leq \dfrac{2C^2\rho^{2t}}{(N+1)p_i},
\end{aligned}
\end{equation}
where the first inequality follows from Cauchy--Schwarz inequality. By the total expectation formula, it follows that
\begin{equation}
\E[\|\vmu_{t+1}(x_i) - \vmu_t(x_i)\|_2^2]\leq  \dfrac{2C^2\rho^{2t}}{(N+1)p_i}.
\end{equation}
Applying the Cauchy–Schwarz inequality gives
\begin{equation}
\E[\|\vmu_{t+1}(x_i) - \vmu_t(x_i)\|_2]\leq  \dfrac{\sqrt{2}C\rho^{t}}{\sqrt{(N+1)p_i}}.
\end{equation}
Summing over \(t=0\) to \(\infty\), we conclude that the expected fidelity
\begin{equation}
\E[F_t(x_i)] = \E[\|\vmu_{\infty}(x_i) - \vmu_0(x_i)\|_2]\leq\dfrac{\sqrt{2}C}{\sqrt{(N+1)p_i}\cdot(1-\rho)},
\end{equation}
which completes the proof.
\end{proof}

%%%%%%%%%%%%%%%%%%%%%%%%%%%%%%%%%%%
\subsection{Proofs of Theorems in Section~4}
\label{subsec:proofs_of_theorems_in_acceleration}

\ArbitrarilySmallConvergenceLargeCovariance*
\begin{proof}
Recall that in~\cref{thm:text_model_diversity_frozen_image_model} the recursion is given by
\begin{equation}
\E[H_{t+1}(\vp_{t+1})|\vp_t] = \Big(1-\frac{1}{N}\Big)\cdot H_t(\vp_t) + \frac{1}{N}\cdot\Big(1-\sum_{i=1}^{K}\E_{\vy\sim r_t}[Z_i(\vy)^2]\Big),
\end{equation}
where
\begin{equation}
Z_i(\vy) = \frac{p_t(x_i)q(\vy|x_i)}{\sum_{k=1}^{K} p_t(x_k)q(\vy|x_k)}.
\end{equation}
Defining
\begin{equation}
\Delta_t \coloneqq H_t(\vp_t) - \E[H_{t+1}(\vp_{t+1})|\vp_t] = \frac{1}{N}\cdot\Big(H_t(\vp_t) - \Big(1-\sum_{i=1}^{K}\E_{\vy\sim r_t}[Z_i(\vy)^2]\Big)\Big),
\end{equation}
we show that for any \(\varepsilon>0\) there exists a family of conditional distributions \(q(\vy|x_i)\) such that
\begin{equation}
\Delta_t < \varepsilon H_t(\vp_t).
\end{equation}
In the limiting scenario where the \(q(\vy|x_i)\)'s are independent of \(i\), we would have \(\E_{\vy\sim r_t}[Z_i(\vy)^2] = p_t(x_i)^2\), so that
\begin{equation}
1-\sum_{i=1}^{K}\E_{\vy\sim r_t}[Z_i(\vy)^2] = 1-\sum_{i=1}^{K} p_t(x_i)^2 = H_t(\vp_t).
\end{equation}
In that case, \(\Delta_t=0\). Now, consider choosing 
\begin{equation}
q(\vy|x_i) = \mathcal{N}(\vy; \vmu_i,\sigma^2 \mI),
\end{equation}
and then taking \(\sigma\) arbitrarily large. In the limit as \(\sigma\to\infty\) the densities \(q(\vy|x_i)\) become almost flat (up to normalization) over \(\vy\). Consequently, the ratio defining \(Z_i(\vy)\) becomes nearly independent of \(\vy\) and approaches
\begin{equation}
Z_i(\vy) \xrightarrow[]{\sigma\rightarrow\infty} \frac{p_t(x_i)}{\sum_{k=1}^{K} p_t(x_k)} = p_t(x_i).
\end{equation}
By changing the variable \(\vy\mapsto\sigma\vy\) and applying the dominated convergence theorem, we obtain
\begin{equation}
\lim_{\sigma\to\infty} \E_{\vy\sim r_t}[Z_i(\vy)^2] = p_t(x_i)^2.
\end{equation}
Hence,
\begin{equation}
\lim_{\sigma\to\infty} \Big(1-\sum_{i=1}^{K}\E_{\vy\sim r_t}[Z_i(\vy)^2]\Big) = H_t(\vp_t).
\end{equation}
Thus, for any \(\delta>0\) (in particular, choose \(\delta=N\varepsilon H_t(\vp_t)\)) there exists \(\sigma_0>0\) such that for all \(\sigma>\sigma_0\) we have
\begin{equation}
1-\sum_{i=1}^{K}\E_{\vy\sim r_t}[Z_i(\vy)^2] - H_t(\vp_t) < N\varepsilon H_t(\vp_t).
\end{equation}
It then follows that
\begin{equation}
\Delta_t = \frac{1}{N}\cdot\Big(H_t(\vp_t) - \Big(1-\sum_{i=1}^{K}\E_{\vy\sim r_t}[Z_i(\vy)^2]\Big)\Big) < \varepsilon H_t(\vp_t).
\end{equation}
This completes the proof.
\end{proof}

\AcceleratdTextModelCollapse*
\begin{proof}
We assume, without loss of generality, that \(\vp_t > \vzero\), or equivalently, we restrict our analysis to texts with nonzero probabilities. By~\cref{thm:text_model_diversity_frozen_image_model}, the diversity reduction in the text model from macro time step \(t\) to \((t+1)\) is given by
\begin{equation}
\begin{aligned}
\Delta_t &= H_t(\vp_t) - \E\big[H_{t+1}(\vp_{t+1}) \big| \{p_t(x_k), \mSigma_{t,0}(x_k)\colon 1\le k\le K\}\big]\\
&= \frac{1}{N}\Big(\sum_{i=1}^{K} \E_{\vy\sim r_t}\big[Z_i(\vy)^2\big] - \sum_{i=1}^{K} p_t(x_i)^2\Big),
\end{aligned}
\end{equation}
where the expectation \(\E_{\vy\sim r_t}[\cdot]\) is with respect to the mixture density
\begin{equation}
r_t(\vy) = \sum_{k=1}^{K} p_t(x_k) q_{t,N_t}(\vy|x_k),
\end{equation}
with
\begin{equation}
q_{t,N_t}(\vy|x_i)=\mathcal{N}(\vy;\vmu_{t,N_t}(x_i),\mSigma_{t,N_t}(x_i))
\end{equation}
denoting the image model after \(N_t\) inner updates. For brevity, hereafter we write \(q_t(\vy|x_i)\) in place of \(q_{t,N_t}(\vy|x_i)\). For each text \(x_i\), we lower-bound \(\E_{\vy\sim r_t}[Z_i(\vy)^2]\) by restricting the integration to a high-probability region of \(q_t(\vy|x_i)\). Define the Mahalanobis ball
\begin{equation}
\mathcal{B}_M(\vmu_t(x_i), r_0) \coloneqq \big\{\vy\in\mathbb{R}^d \colon (\vy-\vmu_t(x_i))^\top \mSigma_{t}(x_i)^{-1} (\vy-\vmu_t(x_i)) \le r_0^2\big\}.
\end{equation}
Then for any prescribed \(\varepsilon'\in(0,1)\) there exists \(r_0>0\) (which can be independent of \(\varepsilon'\)) such that
\begin{equation}
\int_{\mathcal{B}_M(\vmu_t(x_i), r_0)} q_t(\vy|x_i) \dif\vy \ge 1-\varepsilon'.
\end{equation}
Now, for \(\vy\in\mathcal{B}_M(\vmu_t(x_i), r_0)\), by definition of the Mahalanobis ball we have
\begin{equation}
\|\vy-\vmu_t(x_i)\|_2 \le r_0\cdot\lambda_{\max}\big(\mSigma_{t}(x_i)^{1/2}\big).
\end{equation}
By the assumption on the exponential decay of the covariance, if \(N_t\) is large enough then
\begin{equation}
\E\big[\tr\big(\mSigma_{t}(x_i)^{1/2}\big) \big| \{p_t(x_k),\mSigma_{t,0}(x_k)\}\big] \le \tr(\mSigma_{t,0}(x_i)^{1/2})\cdot\rho^{N_t}.
\end{equation}
Thus, for sufficiently large \(N_t\), we may ensure that
\begin{equation}
\|\vy-\vmu_t(x_i)\|_2 \le \frac{\Gamma}{2}
\end{equation}
with probability approaching \(1\) as \(N_t\) increases. We may absorb such a high probability event into the choice of \(\varepsilon'\) for brevity. Because of the mean separation assumption, for any \(x_j\) with \(j\neq i\) we have
\begin{equation}
\|\vmu_t(x_j)-\vmu_t(x_i)\|_2 \ge \Gamma.
\end{equation}
Thus, for all \(\vy\in\mathcal{B}_M(\vmu_t(x_i), r_0)\) it follows by the triangle inequality that
\begin{equation}
\|\vy-\vmu_t(x_j)\|_2 \ge \|\vmu_t(x_j)-\vmu_t(x_i)\|_2 - \|\vy-\vmu_t(x_i)\|_2 \ge \Gamma - \frac{\Gamma}{2} = \frac{\Gamma}{2}.
\end{equation}
From standard Gaussian density properties, one can show that for all large enough \(N_t\),
\begin{equation}
q_t(\vy|x_j) \le \varepsilon'\cdot q_t(\vy|x_i),\quad\text{for all } j\neq i \text{ and for all } \vy\in\mathcal{B}_M(\vmu_t(x_i), r_0).
\end{equation}
Now, recall the definition of the posterior probability
\begin{equation}
Z_i(\vy) = \frac{p_t(x_i)q_t(\vy|x_i)}{\sum_{k=1}^K p_t(x_k)q_t(\vy|x_k)}.
\end{equation}
Using the above bound for \(j\neq i\), for any \(\vy\in\mathcal{B}_M(\vmu_t(x_i), r_0)\) we have
\begin{equation}
\sum_{k\neq i} p_t(x_k)q_t(\vy|x_k) \le \varepsilon'\sum_{k\neq i} p_t(x_k)q_t(\vy|x_i)
= \varepsilon'(1-p_t(x_i))q_t(\vy|x_i).
\end{equation}
Thus,
\begin{equation}
Z_i(\vy) \ge \frac{p_t(x_i)q_t(\vy|x_i)}{p_t(x_i)q_t(\vy|x_i) + \varepsilon'(1-p_t(x_i))q_t(\vy|x_i)}
=\frac{p_t(x_i)}{p_t(x_i) + \varepsilon'(1-p_t(x_i))}.
\end{equation}
Define
\begin{equation}
L\big(p_t(x_i)\big) \coloneqq \frac{p_t(x_i)}{p_t(x_i) + \varepsilon'(1-p_t(x_i))}.
\end{equation}
Note that for each \(i\) and for \(\vy\in\mathcal{B}_M(\vmu_t(x_i), r_0)\) we have
\begin{equation}
Z_i(\vy)^2 \ge L\big(p_t(x_i)\big)^2.
\end{equation}
Now, since the mixture density is defined by
\begin{equation}
r_t(\vy)=\sum_{k=1}^{K} p_t(x_k)q_t(\vy|x_k).
\end{equation}
Taking the expectation of \(Z_i(\vy)^2\) with respect to \(\vy\sim r_t\) gives
\begin{equation}
\E_{\vy\sim r_t}[Z_i(\vy)^2]
=\sum_{k=1}^{K} p_t(x_k)\cdot\E_{q_t(\vy|x_k)}[Z_i(\vy)^2].
\end{equation}
Since the bound \(Z_i(\vy)^2\ge L\big(p_t(x_i)\big)^2\) holds only when \(\vy\in \mathcal{B}_M(\vmu_t(x_i), r_0)\), we restrict to this region. By construction, the Mahalanobis ball \(\mathcal{B}_M(\vmu_t(x_i), r_0)\) has probability at least \(1-\varepsilon'\) under \(q_t(\vy|x_i)\). Hence,
\begin{equation}
\E_{q_t(\vy|x_i)}[Z_i(\vy)^2] \ge (1-\varepsilon')L\big(p_t(x_i)\big)^2.
\end{equation}
Since the mixture density \(r_t(\vy)\) assigns weight \(p_t(x_i)\) to the component \(q_t(\vy|x_i)\), we deduce that
\begin{equation}
\E_{\vy\sim r_t}[Z_i(\vy)^2] \ge p_t(x_i)\cdot (1-\varepsilon')\cdot L\big(p_t(x_i)\big)^2.
\end{equation}
Finally, summing this bound over all \(i\) yields
\begin{equation}
\sum_{i=1}^{K}\E_{\vy\sim r_t}\big[Z_i(\vy)^2\big] \ge \sum_{i=1}^{K} p_t(x_i)\cdot(1-\varepsilon')\cdot L\big(p_t(x_i)\big)^2.
\end{equation}
One may verify that there exists a constant \(\beta>0\) (which depends on \(\varepsilon'\) and the range of \(p_t(x_i)\) values) such that
\begin{equation}
p_t(x_i)\cdot (1-\varepsilon')\cdot L\big(p_t(x_i)\big)^2 \ge p_t(x_i)^2 + \beta p_t(x_i)\big(1-p_t(x_i)\big).
\end{equation}
Summing over \(i\) and using the identity
\begin{equation}
\sum_{i=1}^{K} p_t(x_i)\cdot\big(1-p_t(x_i)\big)= 1-\sum_{i=1}^{K}p_t(x_i)^2= H_t(\vp_t),
\end{equation}
we obtain
\begin{equation}
\sum_{i=1}^{K}\E_{\vy\sim r_t}[Z_i(\vy)^2] \ge \sum_{i=1}^{K} p_t(x_i)^2 + \beta H_t(\vp_t).
\end{equation}
Substituting this back into the expression for \(\Delta_t\),
\begin{equation}
\Delta_t \ge \frac{\beta}{N}H_t(\vp_t).
\end{equation}
In our argument, the constant \(\beta\) depends on \(\varepsilon'\) and on the behavior of the function \(L(p)=p/(p+\varepsilon'(1-p))\) for \(p\in(0,1]\). In particular, if the minimum of \(p_t(x_i)\) is bounded away from \(0\), then \(\beta\) can be taken as a positive constant independent of \(i\). For large \(N_t\) (i.e. when the covariances are sufficiently small), one can make \(\beta\) arbitrarily close to \(1\). This leads to the desired conclusion that
\begin{equation}
\Delta_t > \frac{1-\varepsilon}{N}\cdot H_t(\vp_t)
\end{equation}
for any prescribed \(\varepsilon>0\).
\end{proof}

%%%%%%%%%%%%%%%%%%%%%%%%%%%%%%%%%%%
\subsection{Proofs of Theorems in Section~5}
\label{subsec:proofs_of_theorems_in_matthew}

\DiffConvergenceRateImageModel*
\begin{proof}
For a fixed text \(x_i\), from the proof of~\cref{thm:image_model_covariance_frozen_text_model}, we know that
\begin{equation}
\rho = \E_{N_i\sim\text{Binomial}(N, p_i)}[\rho(N_i)],
\end{equation}
where \(\rho(N_i)\) is the largest eigenvalue of \(\E\big[(\mW/(N_i-1))^{1/2}\big]\). Here \(\mW\sim\text{Wishart}_{d}(\mI, N_i-1)\) is a \(d\)-dimensional Wishart random matrix with \((N_i-1)\) degrees of freedom and identity scale matrix. For \(N_i\leq 1\), \(\rho(N_i)\) is set to \(1\).

We now compute \(\rho(N_i)\) for \(N_i\geq 2\). The Wishart distribution with scale matrix \( \mI_d \) is invariant under orthogonal conjugation, implying that any function \( f(\mW) \) that commutes with orthogonal transformations yields an expected value proportional to the identity matrix. Specifically, since the matrix square root satisfies \( (\mQ\mW\mQ^\top)^{1/2} = \mQ\mW^{1/2}\mQ^\top \) for any orthogonal matrix \( \mQ \), it follows that \( \E[\mW^{1/2}] = \alpha \mI_d \) for some scalar \(\alpha>0\).

Since the square-root function is Fr\'echet differentiable on the cone of positive definite matrices, we can expand
\begin{equation}
f(\mW) = \mW^{1/2}
\end{equation}
around its mean \((N_i-1)\mI_d\). Writing
\begin{equation}
\mW = (N_i-1)\mI_d + \mE,
\end{equation}
with \(\mE = \mW - (N_i-1)\mI_d\) (and \(\E[\mE]=0\)), the Taylor expansion up to second order is
\begin{equation}
\mW^{1/2} = \big((N_i-1)\mI_d + \mE\big)^{1/2} = \sqrt{N_i-1}\mI_d + \frac{1}{2\sqrt{N_i-1}}\mE - \frac{1}{8(N_i-1)^{3/2}}\mE^2 + \mR,
\end{equation}
where \(\mR\) is the remainder term that involves the third Fr\'echet derivative. Taking expectations and noting that \(\E[\mE] = \vzero\) we obtain
\begin{equation}
\E[\mW^{1/2}] = \sqrt{N_i-1}\mI_d - \frac{1}{8(N_i-1)^{3/2}}\E[\mE^2] + \E[\mR].
\end{equation}
Because the Wishart distribution is invariant under orthogonal conjugation, the second moment \(\E[\mE^2]\) is proportional to the identity. In fact, one can show that~\cite{gupta2018matrix}
\begin{equation}
\E[\mE^2] = (d+1)(N_i-1)\mI_d.
\end{equation}
Substituting this into the expansion gives
\begin{equation}
\E[\mW^{1/2}] = \sqrt{N_i-1}\mI_d - \frac{d+1}{8\sqrt{N_i-1}}\mI_d + \E[\mR].
\end{equation}
Dividing by \(\sqrt{N_i-1}\) we obtain
\begin{equation}
\E\Big[\Big(\frac{\mW}{N_i-1}\Big)^{1/2}\Big] = \mI_d - \frac{d+1}{8(N_i-1)}\mI_d + \frac{\E[\mR]}{\sqrt{N_i-1}}.
\end{equation}
Since the matrix on the left is a scalar multiple of the identity, its largest eigenvalue is exactly the scalar factor in front of \(\mI_d\). That is,
\begin{equation}
\rho(N_i) = 1 - \frac{d+1}{8(N_i-1)} + \varepsilon,
\end{equation}
where the error \(\varepsilon\) arises from the remainder term.

We then derive the bound of \(\varepsilon\). The third derivative of \(f(x)=\sqrt{x}\) is
\begin{equation}
f'''(x)=\frac{3}{8}x^{-5/2}.
\end{equation}
In the matrix case, the third Fr\'echet derivative \(D^3\sqrt{\mA}\) has norm of order \(O\big((N_i-1)^{-5/2}\big)\), while \(\|\mE\|\) is of magnitude \(O\big((N_i-1)^{1/2}\big)\). An integral form for the remainder is given by
\begin{equation}
\mR = \int_0^1 \frac{(1-t)^2}{2}D^3\sqrt{(N_i-1)\mI_d+t\mE}[\mE,\mE,\mE]\dif t.
\end{equation}
Taking norms, we obtain
\begin{equation}
\|\mR\|_2 \leq \frac{1}{6}\sup_{t\in[0,1]}\big\lVert D^3\sqrt{(N_i-1)\mI_d+t\mE]}\big\rVert_{\mathrm{op}}\cdot\|\mE\|_2^3.
\end{equation}
Since \(\big\lVert D^3\sqrt{(N_i-1)\mI_d+t\mE]}\big\rVert_{\mathrm{op}}\) is \(O\big((N_i-1)^{-5/2}\big)\) and \(\|\mE\|_2^3\) is \(O\big((N_i-1)^{3/2}\big)\), we deduce that
\begin{equation}
\dfrac{\|\mR\|}{\sqrt{N_i-1}} = O\Big(\frac{1}{(N_i-1)^{3/2}}\Big).
\end{equation}
Thus, there exists a constant \(K>0\) (depending on \(d\)) such that
\begin{equation}
\Big|\rho(N_i) - \Big(1-\frac{d+1}{8(N_i-1)}\Big)\Big|\le \frac{K}{(N_i-1)^{3/2}}.
\end{equation}
In summary, for \(N_i\geq 2\), we have shown via the Taylor expansion that
\begin{equation}
\rho(N_i) = 1 - \frac{d+1}{8(N_i-1)} + O\Big(\frac{1}{(N_i-1)^{3/2}}\Big).
\end{equation}

Finally, we have that
\begin{equation}
\begin{aligned}
\rho &= \E_{N_i\sim\text{Binomial}(N, p_i)}[\rho(N_i)]\\
&= \sum_{N_i=0}^{N}\rho(N_i)\binom{N}{N_i}p_i^{N_i}(1-p_i)^{N-N_i}\\
&=  \sum_{N_i=0}^{N}\Big(1 - \frac{d+1}{8(N_i-1)}\Big)\binom{N}{N_i}p_i^{N_i}(1-p_i)^{N-N_i} + O\Big(\frac{1}{(N-1)^{3/2}}\Big)\\
&\approx 1-\dfrac{d+1}{8(N+1)p_i},
\end{aligned}
\end{equation}
which completes the proof.
\end{proof}

\MatthewEffect*
\begin{proof}
By the assumed convergence rate in~\cref{thm:diff_convergence_rate_image_model}, we have for each \(x_i\),
\begin{equation}
\rho(x_i) \approx 1-c\cdot p_t(x_i)^{-1},
\end{equation}
where we set \(c \coloneqq (d+1)/(8(N+1))\). Hence, the ratio of convergence for the dominant text \(x_1\) and the rarest text \(x_K\) is 
\begin{equation}
\dfrac{\rho(x_1)}{\rho(x_K)} = \dfrac{1-c\cdot p_t(x_1)^{-1}}{1-c\cdot p_t(x_K)^{-1}} \geq 1+c\cdot\big(p_t(x_K)^{-1}-p_t(x_1)^{-1}\big),
\end{equation}
which can be shown via a first-order Taylor expansion. We now recast and relax the optimization problem as
\begin{equation}
\begin{aligned}
\min\quad & p_t(x_K)^{-1}-p_t(x_1)^{-1},\\
\text{s.t.}\quad & p_t(x_1)^2+\dotsb+p_t(x_K)^2 \leq 1-\varepsilon,\\
& p_t(x_1)+\dotsb+p_t(x_K)=1,\\
& p_t(x_1)\geq \dotsb\geq p_t(x_K)>0.
\end{aligned}
\end{equation}
By convexity of the square function, for fixed \(p_t(x_1)\) the sum \(\sum_{i=2}^{K}p_t(x_i)^2\) is minimized (and hence the gap between the reciprocal of the smallest and the largest probabilities is maximized) when the remaining \((K-1)\) probabilities are equal. Thus, set
\begin{equation}
p_t(x_1)=1-\delta,\quad p_t(x_2)=\dotsb=p_t(x_K)=\frac{\delta}{K-1},
\end{equation}
with \(\delta\in(0, 1)\) given by the solution to
\begin{equation}
(1-\delta)^2+(K-1)\cdot\Big(\frac{\delta}{K-1}\Big)^2 = (1-\delta)^2+\frac{\delta^2}{K-1}=1-\varepsilon.
\end{equation}
This quadratic equation in \(\delta\) has solution
\begin{equation}
\delta = \dfrac{K-1}{K}\cdot\bigg(1-\sqrt{1-\dfrac{K\varepsilon}{K-1}}\bigg),
\end{equation}
where we take the minus sign in the quadratic formula so that \(\delta\in(0, 1)\). With this choice we have
\begin{equation}
p_t(x_1)^{-1}=\frac{1}{1-\delta},\quad p_t(x_K)^{-1}=\frac{K-1}{\delta}.
\end{equation}
Thus,
\begin{equation}
\begin{aligned}
p_t(x_K)^{-1}-p_t(x_1)^{-1} &= \frac{K-1}{\delta} - \frac{1}{1-\delta}\\
&\geq \dfrac{K-1}{\delta}-1\\
&=\dfrac{K}{1-\sqrt{1-K\varepsilon/(K-1)}}-1\\
&\geq \dfrac{K-1}{\varepsilon}-1.
\end{aligned}
\end{equation}
Considering that \(c\) is positive and typically small compared to \(1\), the additive ``\(-1\)'' term is negligible when \(\varepsilon\) is small. In any event, one then obtains the desired inequality
\begin{equation}
\frac{\rho(x_1)}{\rho(x_K)} \geq 1+c\cdot\Big(\dfrac{K-1}{\varepsilon}-1\Big)\geq \max\Big(\dfrac{(d+1)(K-1)}{8(N+1)}\cdot\varepsilon^{-1}, 1\Big).
\end{equation}
This completes the proof.
\end{proof}

%%%%%%%%%%%%%%%%%%%%%%%%%%%%%%%%%%%
\subsection{Proofs of Theorems in Section~6}
\label{subsec:proofs_of_theorems_in_stabilization}

\StabilizationTextModelViaInjection*
\begin{proof}
Conditioning on the current text probability vector \(\vp_t\), we consider two cases. (\romannumeral 1) If no injection occurs at macro time step \(t\) (which happens with probability \(1-\alpha\)), then, by~\cref{cor:lower_bound_text_model_diversity}, the text update satisfies
\begin{equation}
\E[H_{t+1}|\vp_t,\text{ no injection}] \geq \Big(1-\frac{1}{N}\Big)\cdot H_t.
\end{equation}
(\romannumeral 2) When an injection occurs (with probability \(\alpha\)), the text model updates by reallocating a fraction \(\varepsilon\) of the probability mass from each existing text to a new text \(x_{\mathrm{new}}\). That is, for \(1\le i\le K\),
\begin{equation}
\widetilde{p}_t(x_i) = (1-\varepsilon)\cdot p_t(x_i),
\end{equation}
and we set
\begin{equation}
\widetilde{p}_t(x_{\mathrm{new}}) = \varepsilon.
\end{equation}
The diversity immediately after injection is given by
\begin{equation}
H_{\text{inj}} = 1 - \Big((1-\varepsilon)^2\sum_{i=1}^{K} p_t(x_i)^2 + \varepsilon^2\Big).
\end{equation}
Since \(\sum_{i=1}^{K} p_t(x_i)^2 = 1 - H_t\), we can rewrite this as
\begin{equation}
H_{\text{inj}} = 1 - (1-\varepsilon)^2 (1-H_t) - \varepsilon^2.
\end{equation}
In the worst-case scenario (when \(H_t=0\)), we obtain
\begin{equation}
H_{\text{inj}} \ge 1 - (1-\varepsilon)^2 - \varepsilon^2 = 2\varepsilon - 2\varepsilon^2.
\end{equation}
Following the injection, the subsequent text update (which contracts diversity by at most a factor of \(N^{-1}\)) guarantees that
\begin{equation}
\E[H_{t+1}|\vp_t,\text{ injection}] \ge \Big(1-\frac{1}{N}\Big)\cdot H_{\text{inj}} \ge \Big(1-\frac{1}{N}\Big)\cdot (2\varepsilon - 2\varepsilon^2).
\end{equation}
By the law of total expectation, we obtain
\begin{equation}
\E[H_{t+1}|\vp_t] = (1-\alpha)\E[H_{t+1}|\vp_t,\text{ no injection}] + \alpha\E[H_{t+1}|\vp_t,\text{ injection}].
\end{equation}
Thus,
\begin{equation}
\E[H_{t+1}|\vp_t] \ge (1-\alpha)\Big(1-\frac{1}{N}\Big)\cdot H_t + \alpha\Big(1-\frac{1}{N}\Big)\cdot(2\varepsilon - 2\varepsilon^2)
\end{equation}
Taking expectations over \(\vp_t\) and iterating this inequality gives
\begin{equation}
\liminf_{t\to\infty}\E[H_t] \geq \frac{2\alpha(1-N^{-1})(\varepsilon - \varepsilon^2)}{1-(1-\alpha)(1-N^{-1})}.
\end{equation}
\end{proof}

\StabilizationImageDiversityInjection*
\begin{proof}
For a fixed text \(x_i\), condition on mean vector \(\vmu_t(x_i)\), covariance matrix \(\mSigma_t(x_i)\), and the number \(N_i\) of training samples. Denote by \(\mS_t(x_i)\) (resp.\ \(\mS_{\text{user}}(x_i)\)) the sample covariance computed from the training samples (resp.\ user-content images), and by \(\vm_t(x_i)\) (resp.\ \(\vm_{\text{user}}(x_i)\)) the sample mean computed from the training samples (resp.\ user-content images). Then the combined sample covariance after image injection is given by
\begin{equation}
\label{eq:variance_decomposition}
\begin{aligned}
\mSigma_{t+1}(x_i) =\;& \frac{N_i-1}{N_i+N_0-1}\mS_t(x_i)
+ \frac{N_0-1}{N_i+N_0-1}\mS_{\text{user}}(x_i)\\
&+\frac{N_iN_0}{(N_i+N_0)(N_i+N_0-1)}\big(\vm_t(x_i)-\vm_{\text{user}}(x_i)\big)
\big(\vm_t(x_i)-\vm_{\text{user}}(x_i)\big)^\top.
\end{aligned}
\end{equation}
Since the matrices \(\mS_t(x_i)\) and
\begin{equation}
\big(\vm_t(x_i)-\vm_{\text{user}}(x_i)\big)
\big(\vm_t(x_i)-\vm_{\text{user}}(x_i)\big)^\top
\end{equation}
are positive semidefinite, we have
\begin{equation}
\mSigma_{t+1}(x_i) \succeq \frac{N_0-1}{N_i+N_0-1}\mS_{\text{user}}(x_i).
\end{equation}
Moreover, since \(N_i\le N\), it follows that
\begin{equation}
\frac{N_0-1}{N_i+N_0-2} \ge \frac{N_0-1}{N+N_0-1},
\end{equation}
so that
\begin{equation}
\mSigma_{t+1}(x_i) \succeq \frac{N_0-1}{N+N_0-1}\mS_{\text{user}}(x_i).
\end{equation}
Because the matrix square root is operator monotone, applying it to both sides yields
\begin{equation}
\mSigma_{t+1}(x_i)^{1/2} \succeq \sqrt{\frac{N_0-1}{N+N_0-1}}\;\mS_{\text{user}}(x_i)^{1/2}.
\end{equation}
Taking the expectation (using the law of total expectation) gives
\begin{equation}
\E[\mSigma_{t+1}(x_i)^{1/2}] \succeq \sqrt{\frac{N_0-1}{N+N_0-1}}\E[\mS_{\text{user}}(x_i)^{1/2}].
\end{equation}
By definition of the user sample covariance, we can write
\begin{equation}
\mS_{\text{user}}(x_i)^{1/2} = \frac{1}{N_0-1}\cdot\mSigma_{\text{user}}(x_i)^{1/2}\cdot\mW^{1/2}\cdot\mSigma_{\text{user}}(x_i)^{1/2},
\end{equation}
where \(\mW\sim\text{Wishart}_d(I,N_0-1)\) is a \(d\)-dimensional Wishart matrix with \((N_0 - 1)\) degrees of freedom and identity scale matrix. Taking expectations and using the orthogonal conjugation invariance of the Wishart distribution, we have
\begin{equation}
\E[\mW^{1/2}] = \alpha\mI_d,
\end{equation}
which implies
\begin{equation}
\E[\mS_{\text{user}}(x_i)^{1/2}] = \frac{\alpha}{N_0-1}\mSigma_{\text{user}}(x_i)^{1/2}.
\end{equation}
Thus,
\begin{equation}
\E[\mSigma_{t+1}(x_i)^{1/2}]
\succeq \sqrt{\frac{N_0-1}{N+N_0-1}}\cdot \frac{\alpha}{N_0-1}\mSigma_{\text{user}}(x_i)^{1/2}.
\end{equation}
Taking the trace on both sides (and using the linearity of both the trace and the expectation), we obtain
\begin{equation}
\E\big[\tr\big(\mSigma_{t+1}(x_i)^{1/2}\big)\big]
\geq \alpha\cdot \frac{1}{\sqrt{(N_0-1)(N+N_0-1)}} \cdot \E\big[\tr\big(\mSigma_{\text{user}}(x_i)^{1/2}\big)\big],
\end{equation}
which completes the proof.
\end{proof}

\BoundednessImageFidelityInjection*
\begin{proof}
For a fixed text \(x_i\), recall that at macro time step \(t\) the following holds. First, \(Np_i\) generated images are drawn from the distribution \(\mathcal{N}\big(\vmu_t(x_i),\mSigma_t(x_i)\big)\). Their sample mean is
\begin{equation}
\vm_t(x_i) = \vmu_t(x_i) + \vvarepsilon_{1,t},
\end{equation}
where, conditioned on the current parameters \(\vmu_t(x_i)\) and \(\mSigma_t(x_i)\), we have
\begin{equation}
\E[\vvarepsilon_{1,t}|\vmu_t(x_i),\mSigma_t(x_i)]=\vzero, \quad \Cov(\vvarepsilon_{1,t}|\vmu_t(x_i),\mSigma_t(x_i))=\frac{\mSigma_t}{Np_i}.
\end{equation}
Second, \(N_0\) injected images are drawn from \(\mathcal{N}\big(\vmu_0(x_i),\mSigma_0(x_i)\big)\) with sample mean
\begin{equation}
\widetilde{\vmu}_t(x_i) = \vmu_0(x_i) + \vvarepsilon_{2,t},
\end{equation}
where
\begin{equation}
\E[\vvarepsilon_{2,t}]=\vzero, \quad \Cov(\vvarepsilon_{2,t})=\frac{\mSigma_0(x_i)}{N_0}.
\end{equation}
The updated mean is defined by
\begin{equation}
\vmu_{t+1} = \frac{Np_i \vm_t(x_i) + N_0 \widetilde{\vmu}_t(x_i)}{Np_i+N_0}.
\end{equation}
Define the error vector by 
\begin{equation}
\vdelta_t(x_i) \coloneqq \vmu_t(x_i) - \vmu_0(x_i),
\end{equation}
and introduce the constants
\begin{equation}
\lambda \coloneqq \frac{Np_i}{Np_i+N_0},\quad 1-\lambda = \frac{N_0}{Np_i+N_0}.
\end{equation}
Then we can write
\begin{equation}
\begin{aligned}
\vmu_{t+1}(x_i) &= \lambda (\vmu_t(x_i) + \vvarepsilon_{1,t}) + (1-\lambda) (\vmu_0(x_i) + \vvarepsilon_{2,t}),\\
\vdelta_{t+1}(x_i) &= \vmu_{t+1}(x_i) - \vmu_0(x_i) \\
&= \lambda (\vmu_t(x_i) - \vmu_0(x_i)) + \lambda \vvarepsilon_{1,t} + (1-\lambda) \vvarepsilon_{2,t}\\
&= \lambda \vdelta_t(x_i) + \lambda \vvarepsilon_{1,t} + (1-\lambda) \vvarepsilon_{2,t}.
\end{aligned}
\end{equation}
Taking the squared \(\ell_2\) norm and then the conditional expectation (with conditioning on \(\vmu_t(x_i)\) and \(\mSigma_t(x_i)\)), we have
\begin{equation}
\begin{aligned}
\E\big[\|\vdelta_{t+1}(x_i)\|_2^2\big|\vmu_t(x_i),\mSigma_t(x_i)\big] &= \lambda^2\|\vdelta_t(x_i)\|_2^2 + \lambda^2 \E\big[\|\vvarepsilon_{1,t}\|_2^2\big|\vmu_t,\mSigma_t\big] + (1-\lambda)^2 \E[\|\vvarepsilon_{2,t}\|_2^2]\\
&= \lambda^2\|\vdelta_t(x_i)\|_2^2 + \frac{\lambda^2}{Np_i} \tr\big(\mSigma_t(x_i)\big) + \frac{(1-\lambda)^2}{N_0} \tr\big(\mSigma_0(x_i)\big).
\end{aligned}
\end{equation}
Defining the expected covariance trace as
\begin{equation}
S_t(x_i) \coloneqq \E\big[\tr\big(\mSigma_t(x_i)\big)\big],
\end{equation}
the law of total expectation implies
\begin{equation}
\E[F_{t+1}(x_i)^2] = \lambda^2 \E[F_t(x_i)^2] + \frac{\lambda^2}{Np_i} S_t(x_i) + \frac{(1-\lambda)^2}{N_0} S_0(x_i).
\end{equation}
Next, we analyze the covariance dynamics. Under formula~\labelcref{eq:variance_decomposition}, we have (again, conditioning on \(\vmu_t(x_i)\) and \(\mSigma_t(x_i)\)):
\[
\begin{aligned}
\E\big[\tr(\mSigma_{t+1})\big|\vmu_t(x_i),\mSigma_t(x_i)\big] =&\; \lambda \tr(\mSigma_t(x_i)) + (1-\lambda) \tr(\mSigma_0(x_i))\\
&+ \dfrac{Np_iN_0}{(Np_i+N_0)(Np_i+N_0-1)} \|\vdelta_t(x_i)\|_2^2.  
\end{aligned}
\]
Taking total expectation yields
\[
S_{t+1}(x_i) = \lambda S_t(x_i) + (1-\lambda) \tr(\mSigma_0(x_i)) + \dfrac{Np_iN_0}{(Np_i+N_0)(Np_i+N_0-1)}\E[F_t(x_i)^2].
\]
Thus, the coupled recurrences for the expected fidelity and the expected covariance trace are
\[
\begin{cases}
\E[F_{t+1}(x_i)^2] = \lambda^2 \E[F_t(x_i)^2] + \dfrac{\lambda^2}{Np_i} S_t(x_i) + \dfrac{(1-\lambda)^2}{N_0} S_0(x_i),\\
S_{t+1}(x_i) = \lambda S_t(x_i) + (1-\lambda)S_0(x_i) + \dfrac{Np_iN_0}{(Np_i+N_0)(Np_i+N_0-1)} \E[F_t(x_i)^2].
\end{cases}
\]
Standard results on linear systems guarantee that if \(\lambda < 1\) (i.e., if \(N_0>0\)), then the sequences \(\{\E[F_t(x_i)^2]\}\) and \(\{S_t(x_i)\}\) converge to finite limits, which are
\[
\begin{aligned}
\lim_{t\rightarrow\infty}\E[F_t(x_i)^2] &= \dfrac{1-(Np_i)^{-1}\lambda}{Np_i\lambda^{-1}-1-Np_i\lambda}\cdot S_0(x_i).\\
\lim_{t\rightarrow\infty}S_t(x_i) &=\Big(1+\dfrac{\lambda}{Np_i\lambda^{-1}-1-Np_i\lambda}\Big)\cdot S_0(x_i).
\end{aligned}
\]
By Cauchy--Schwarz inequality, we have
\[
\limsup_{t\rightarrow\infty}\E[F_t(x_i)] \leq \Big(\dfrac{1-(Np_i)^{-1}\lambda}{Np_i\lambda^{-1}-1-Np_i\lambda}\cdot \tr\big(\mSigma_0(x_i)\big)\Big)^{1/2},
\]
which completes the proof.
\end{proof}

%%%%%%%%%%%%%%%%%%%%%%%%%%%%%%%%%%%%%%%%%%%%%%%%%%%%
\section{Visualizations of a Typical Run}
\label{sec:visualizations_of_a_typical_run}

In this section, we present a set of representative visualizations that capture the dynamics of our co-evolving generative models under various configurations. For all experiments, we set the text model to have \(K=5\) texts and the image model to operate in a \(d=2\) latent space. The initial mean vectors are evenly distributed on the unit circle, and all covariance matrices are set to the identity matrix. We run the simulations for a total of \(T=1000\) macro time steps and display the text model histogram and corresponding generated image samples every \(250\) macro time steps. Throughout all figures in this section, the top row shows histograms of the text model's probability distribution \(\vp_t=(p(x_1), p(x_2), \dotsc, p(x_K))\), while the bottom row displays the generated samples, with distinct colors indicating samples associated with different texts.

%%%%%%%%%%%%%%%%%%%%%%%%%%%%%%%%%%%
\subsection{Freeze the Image Model and Update the Text Model}
\label{subsec:freeze_the_image_model_and_update_the_text_model}

When the image model is frozen, it provides a constant source of feedback while only the text model is updated. As shown in~\cref{fig:freeze_image_model_and_update_text_model}, the text model gradually loses diversity, and the probability distribution over texts becomes increasingly imbalanced over time. This behavior is consistent with the theoretical predictions presented in~\cref{thm:text_model_diversity_frozen_image_model}.

\begin{figure}[htb]
\centering
\includegraphics[width=1\linewidth]{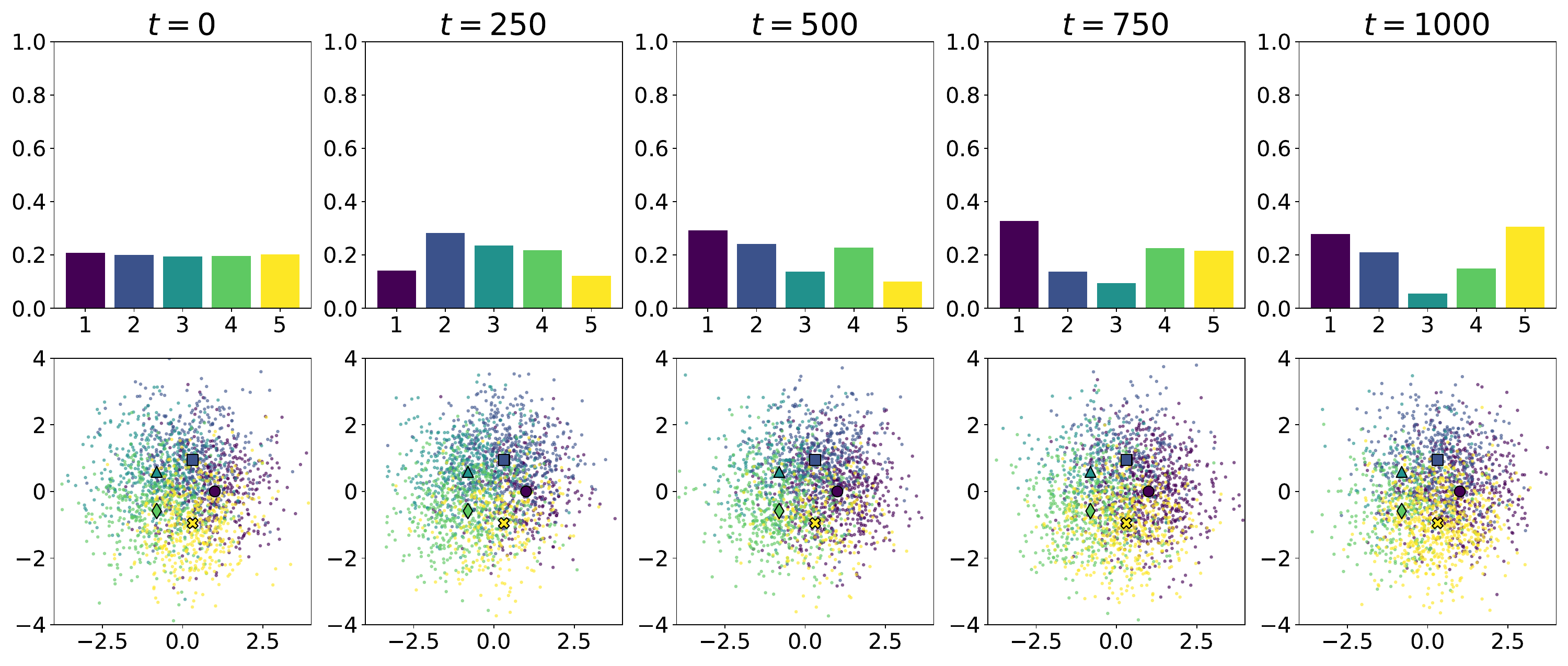}
\caption{Text model diversity collapses under the frozen image model. As the histograms in the first row indicate, the text distribution becomes increasingly imbalanced, illustrating the gradual loss of diversity. This result is in agreement with our theoretical analysis in~\cref{thm:text_model_diversity_frozen_image_model}.}
\label{fig:freeze_image_model_and_update_text_model}
\end{figure}

%%%%%%%%%%%%%%%%%%%%%%%%%%%%%%%%%%%
\subsection{Freeze the Text Model and Update the Image Model}
\label{subsec:freeze_the_text_model_and_update_the_image_model}

When the text model is frozen, \cref{fig:freeze_text_model_and_update_image_model} demonstrates that the image model quickly converges to a narrow output distribution, which verifies~\cref{thm:image_model_covariance_frozen_text_model}. Although the outputs become almost identical (i.e., diversity collapses), the mean vectors remain relatively close to the initial states, indicating that fidelity is preserved despite the loss of diversity. This behavior confirms our analytical predictions regarding the boundedness of image model fidelity under frozen text conditions in~\cref{thm:image_model_mean_frozen_text_model}.

\begin{figure}[htb]
\centering
\includegraphics[width=1\linewidth]{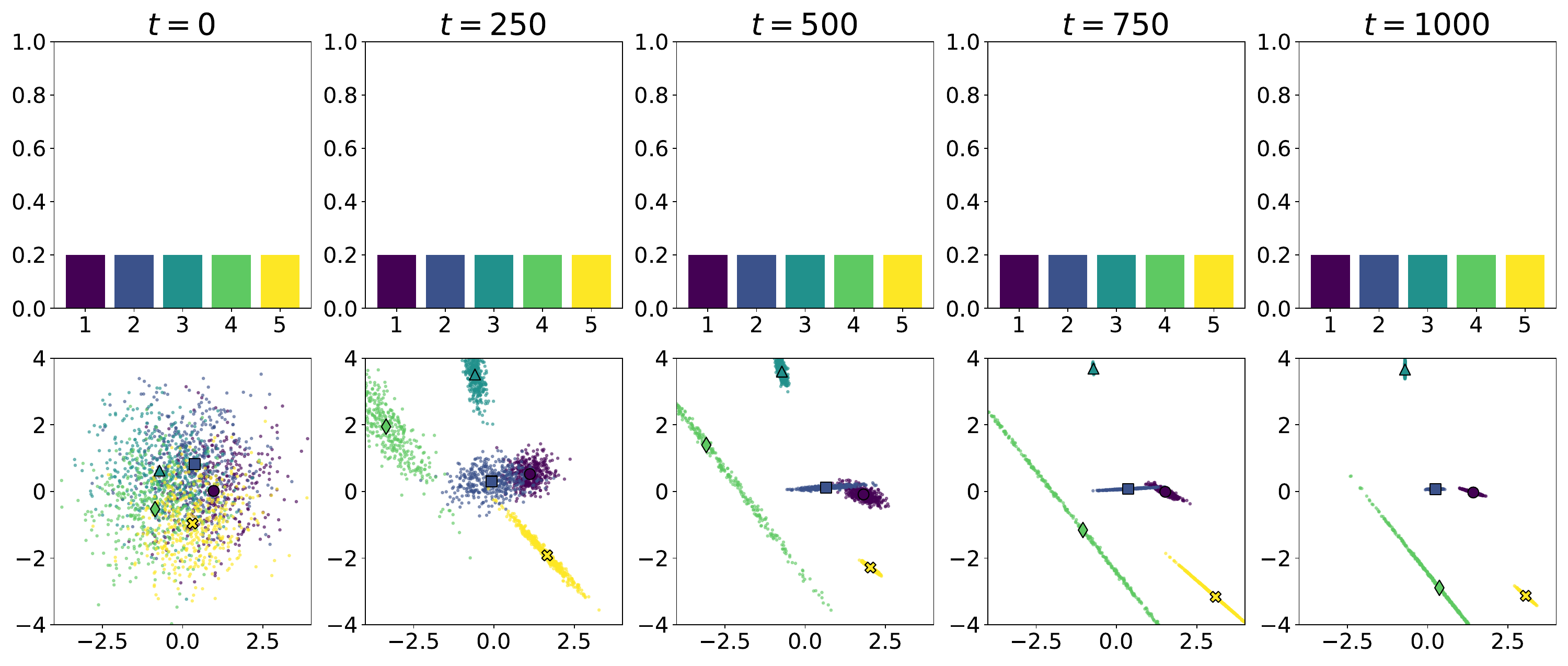}
\caption{Image model diversity collapses while its fidelity remains bounded under the frozen text model. As indicated by the generated samples in the second row, the image model quickly contracts its covariance (diversity), while its mean remains relatively close to the initial state, in agreement with the theoretical results of~\cref{thm:image_model_covariance_frozen_text_model,thm:image_model_mean_frozen_text_model}.}
\label{fig:freeze_text_model_and_update_image_model}
\end{figure}

%%%%%%%%%%%%%%%%%%%%%%%%%%%%%%%%%%%
\subsection{Update Both the Text Model and the Image Model}
\label{subsec:update_both_the_text_model_and_the_image_model}

We then examine the scenario where both the text and image models update concurrently (i.e., \(M_t=1\) and \(N_t=1\) in~\cref{alg:co-evolving_generative_models}). As illustrated in~\cref{fig:update_both_text_model_and_image_model}, the mutual feedback between the two models accelerates the collapse process. 
Specifically, the contracting covariance of the image model sharpens the feedback to the text model, which in turn concentrates its probability mass on a few dominant texts. This self-reinforcing dynamic drives the system rapidly toward a degenerate state. 
Moreover, a Matthew effect emerges in the image model: while images associated with the dominant text maintain relatively higher diversity, those corresponding to less frequent texts collapse even faster. These observations are consistent with our predictions in~\cref{thm:acceleratd_text_model_collapse,thm:matthew_effect}.

\begin{figure}[htb]
\centering
\includegraphics[width=1\linewidth]{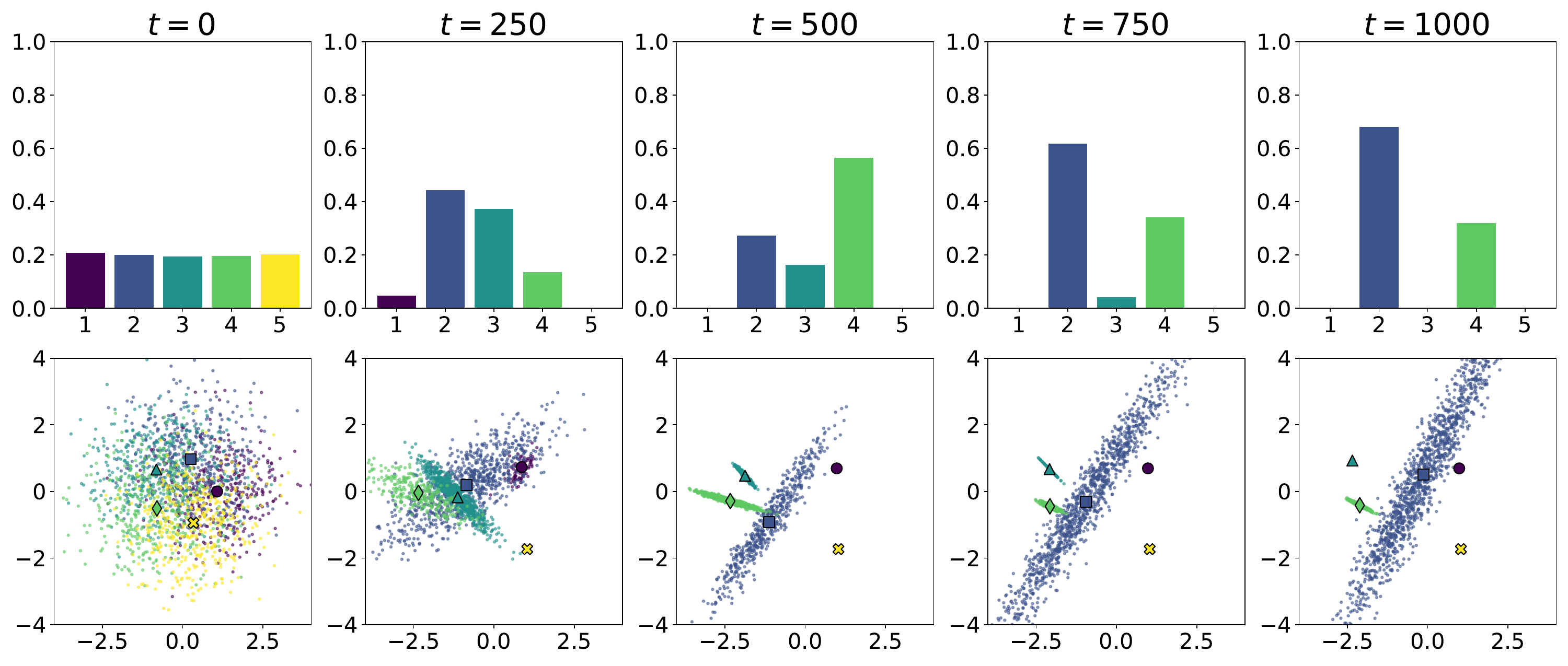}
\caption{Mutual reinforcement accelerates the collapse of text model diversity and induces a Matthew effect in the image model under the co-evolving training procedure. The histograms in the first row show that the text model collapse is accelerated compared with the frozen image model in~\cref{fig:freeze_image_model_and_update_text_model}. As indicated by the generated samples in the second row, the image model rapidly contracts its covariance while the dominant text retains relatively higher diversity, consistent with the theoretical results of~\cref{thm:acceleratd_text_model_collapse,thm:matthew_effect}.}
\label{fig:update_both_text_model_and_image_model}
\end{figure}

%%%%%%%%%%%%%%%%%%%%%%%%%%%%%%%%%%%
\subsection{Stabilize the Text Model via Corpus Injection}
\label{subsec:stabilize_the_text_model_via_corpus_injection}

In this section, we examine how corpus injection stabilizes the text model from collapsing. \Cref{fig:stabilize_text_model_via_corpus_injection} demonstrates that random corpus injection with a fraction \(\varepsilon=0.1\) of the existing probability mass redistributed to newly introduced texts at an injection probability of \(\alpha=0.005\) effectively prevents the text model from collapsing. As shown in the histograms, this reallocation mechanism consistently maintains a non-degenerate, diverse probability distribution over texts, thereby counteracting the self-reinforcing dynamics that would otherwise drive the system toward a one-hot distribution. These results validate the predictions of~\cref{thm:stabilization_text_model_via_injection}.

\begin{figure}[H]
\centering
\includegraphics[width=1\linewidth]{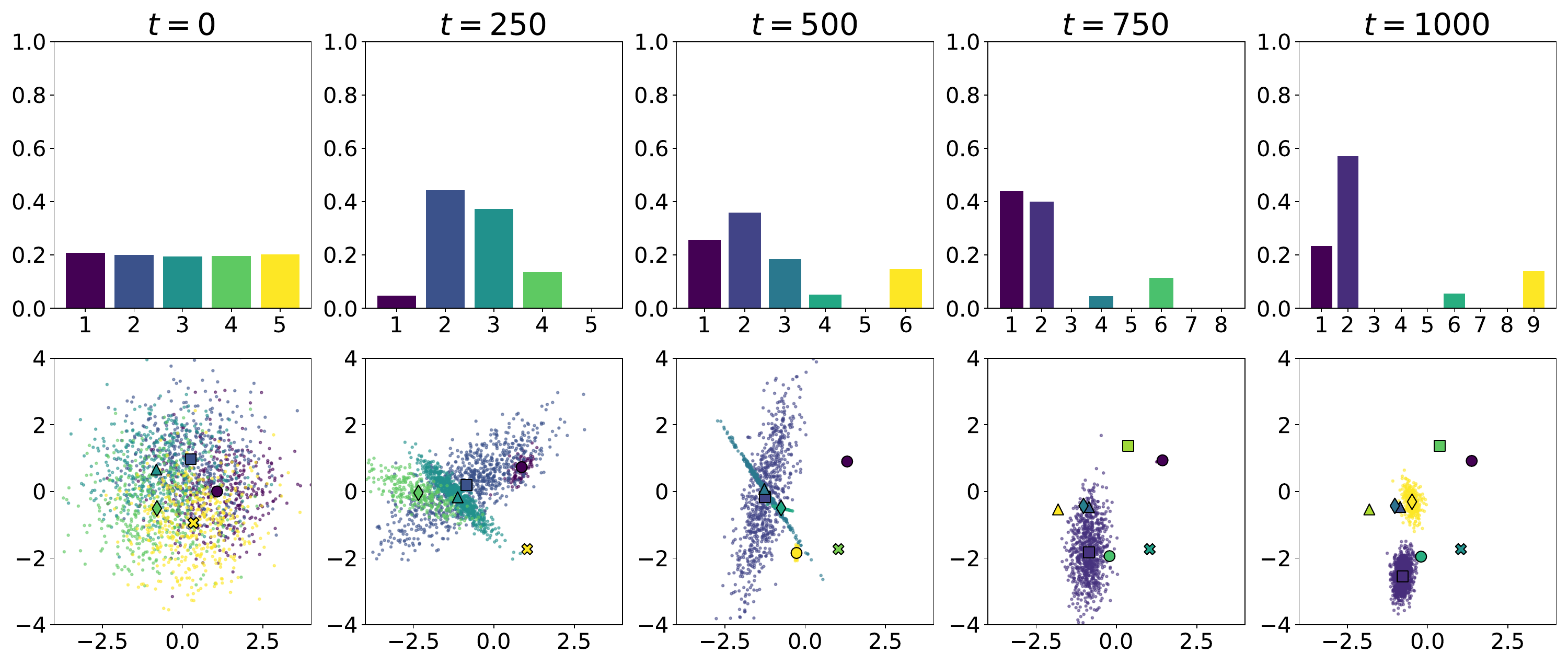}
\caption{Random corpus injection effectively stabilizes the text model. As indicated by the histograms in the first row, the reallocation of probability mass to new texts prevents the collapse into a one-hot distribution and sustains a meaningful level of diversity over time, validating~\cref{thm:stabilization_text_model_via_injection}.}
\label{fig:stabilize_text_model_via_corpus_injection}
\end{figure}

%%%%%%%%%%%%%%%%%%%%%%%%%%%%%%%%%%%
\subsection{Stabilize the Image Model via User-Content Injection}
\label{subsec:stabilize_the_image_model_via_user-content_injection}

In this section, we examine how user-content injection prevents the collapse of the image model. As shown in \cref{fig:stabilize_image_model_via_user-content_injection}, injecting \(N_0=100\) user-generated images into the image model's training process effectively stabilizes the model. This external input stops the exponential contraction of the covariance and keeps the mean close to its initial state, thereby preserving both diversity and fidelity. These results are consistent with the theoretical predictions in \cref{thm:stabilization_image_diversity_injection,thm:boundedness_image_fidelity_injection}.

\begin{figure}[H]
\centering
\includegraphics[width=1\linewidth]{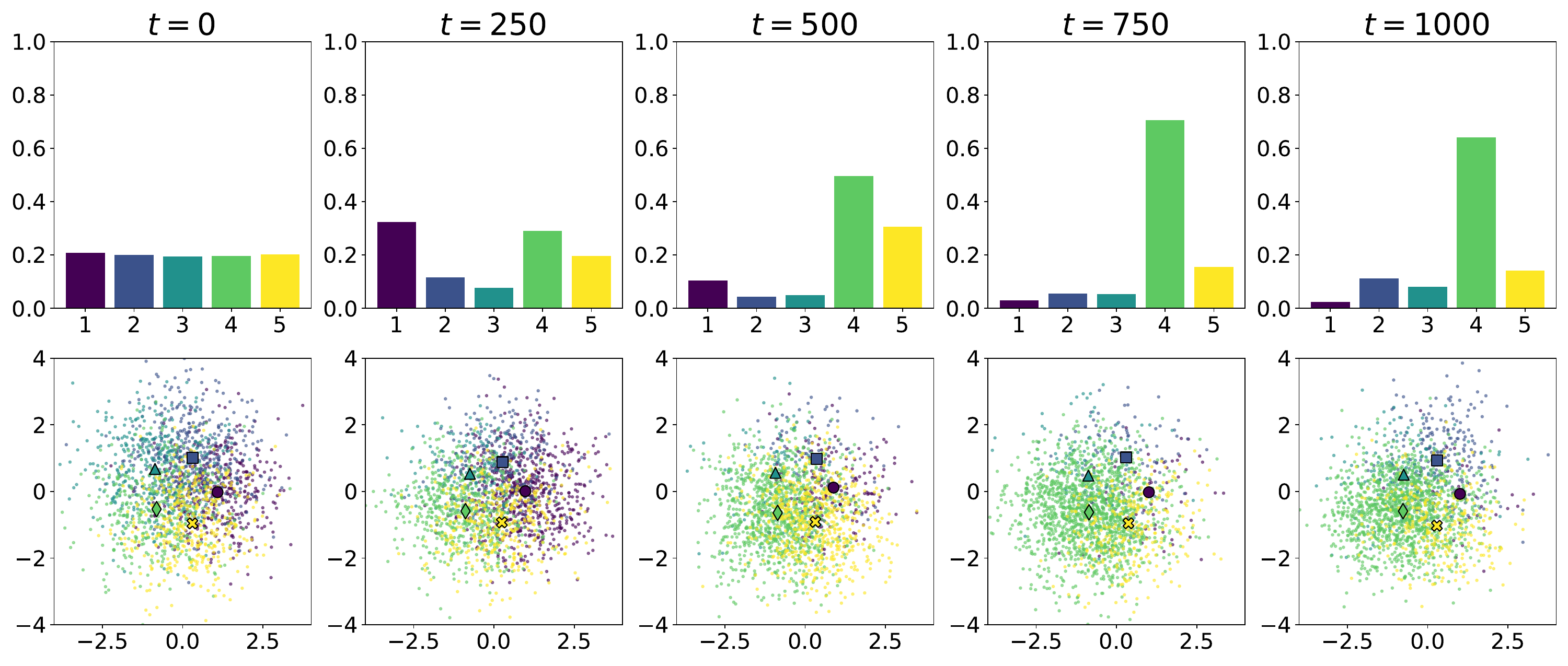}
\caption{User-content injection stabilizes the image model. From the generated samples in the second row, incorporating external images from the initial distribution keeps the image model's covariance bounded and its mean close to the initial target, thereby preserving both diversity and fidelity. This behavior aligns with the predictions in~\cref{thm:stabilization_image_diversity_injection,thm:boundedness_image_fidelity_injection}.}
\label{fig:stabilize_image_model_via_user-content_injection}
\end{figure}

\end{document}